\documentclass[english]{article}
\usepackage[T1]{fontenc}
\usepackage[utf8]{inputenc}
\usepackage{babel}
\usepackage{float}
\usepackage{amsmath}
\usepackage{amsthm}
\usepackage{amssymb}
\usepackage{graphicx}
\usepackage[numbers]{natbib}
\usepackage{xargs}[2008/03/08]
\usepackage{microtype}
\usepackage[unicode=true,
 bookmarks=false,
 breaklinks=false,pdfborder={0 0 1},backref=section,colorlinks=false]
 {hyperref}

\makeatletter
\theoremstyle{plain}
\newtheorem{thm}{\protect\theoremname}
\theoremstyle{definition}
\newtheorem{defn}[thm]{\protect\definitionname}
\theoremstyle{plain}
\newtheorem{prop}[thm]{\protect\propositionname}
\theoremstyle{plain}
\newtheorem{cor}[thm]{\protect\corollaryname}
\theoremstyle{plain}
\newtheorem{lem}[thm]{\protect\lemmaname}

\usepackage[final]{neurips_2021}

\usepackage{url}% simple URL typesetting
\usepackage{booktabs}% professional-quality tables
\usepackage{amsfonts}% blackboard math symbols
\usepackage{nicefrac}% compact symbols for 1/2, etc.
\usepackage{xcolor}% colors
\usepackage[normalem]{ulem}

\author{
    Trung Phung$^{1}$ \hspace{0.04cm}
    Trung Le$^{2}$ \hspace{0.04cm}
    Long Vuong$^{1}$ \hspace{0.04cm} 
    Toan Tran$^{1}$ \hspace{0.04cm}
    Anh Tran$^{1}$ \hspace{0.04cm}
    Hung Bui$^{1}$ \hspace{0.04cm}
    Dinh Phung$^{1,2}$ \\
    $^{1}$ VinAI Research, Vietnam \hspace{0.1cm}
    $^{2}$ Monash University, Australia \\
	\texttt{v.trungpq3@vinai.io, trunglm@monash.edu}, \\ \texttt{\{v.longvt8, v.toantm3, v.anhtt152, v.hungbh1, v.dinhpq2\}@vinai.io}
}

\usepackage{babel}
\providecommand{\definitionname}{Definition}

\providecommand{\theoremname}{Theorem}

\@ifundefined{showcaptionsetup}{}{%
 \PassOptionsToPackage{caption=false}{subfig}}
\usepackage{subfig}
\makeatother

\providecommand{\corollaryname}{Corollary}
\providecommand{\definitionname}{Definition}
\providecommand{\lemmaname}{Lemma}
\providecommand{\propositionname}{Proposition}
\providecommand{\theoremname}{Theorem}

\begin{document}
\title{On Learning Domain-Invariant Representations for Transfer Learning
with Multiple Sources}
\maketitle
\begin{abstract}
Domain adaptation (DA) benefits from the rigorous theoretical works
that study its insightful characteristics and various aspects, e.g.,
learning domain-invariant representations and its trade-off. However,
it seems not the case for the multiple source DA and domain generalization
(DG) settings which are remarkably more complicated and sophisticated
due to the involvement of multiple source domains and potential unavailability
of target domain during training. In this paper, we develop novel
upper-bounds for the target general loss which appeal to us to define
two kinds of domain-invariant representations. We further study the
pros and cons as well as the trade-offs of enforcing learning each
domain-invariant representation. Finally, we conduct experiments to
inspect the trade-off of these representations for offering practical
hints regarding how to use them in practice and explore other interesting
properties of our developed theory. 
\end{abstract}
\global\long\def\sidenote#1{\marginpar{\small\emph{{\color{Medium}#1}}}}%

\global\long\def\se{\hat{\text{se}}}%
\global\long\def\interior{\text{int}}%
\global\long\def\boundary{\text{bd}}%
\global\long\def\ML{\textsf{ML}}%
\global\long\def\GML{\mathsf{GML}}%
\global\long\def\HMM{\mathsf{HMM}}%
\global\long\def\support{\text{supp}}%
\global\long\def\new{\text{*}}%
\global\long\def\stir{\text{Stirl}}%
\global\long\def\mA{\mathcal{A}}%
\global\long\def\mB{\mathcal{B}}%
\global\long\def\mF{\mathcal{F}}%
\global\long\def\mK{\mathcal{K}}%
\global\long\def\mH{\mathcal{H}}%
\global\long\def\mX{\mathcal{X}}%
\global\long\def\mZ{\mathcal{Z}}%
\global\long\def\mS{\mathcal{S}}%
\global\long\def\Ical{\mathcal{I}}%
\global\long\def\mT{\mathcal{T}}%
\global\long\def\Pcal{\mathcal{P}}%
\global\long\def\dist{d}%
\global\long\def\HX{\entro\left(X\right)}%
\global\long\def\entropyX{\HX}%
\global\long\def\HY{\entro\left(Y\right)}%
\global\long\def\entropyY{\HY}%
\global\long\def\HXY{\entro\left(X,Y\right)}%
\global\long\def\entropyXY{\HXY}%
\global\long\def\mutualXY{\mutual\left(X;Y\right)}%
\global\long\def\mutinfoXY{\mutualXY}%
\global\long\def\given{\mid}%
\global\long\def\gv{\given}%
\global\long\def\goto{\rightarrow}%
\global\long\def\asgoto{\stackrel{a.s.}{\longrightarrow}}%
\global\long\def\pgoto{\stackrel{p}{\longrightarrow}}%
\global\long\def\dgoto{\stackrel{d}{\longrightarrow}}%
\global\long\def\lik{\mathcal{L}}%
\global\long\def\logll{\mathit{l}}%
\global\long\def\vectorize#1{\mathbf{#1}}%

\global\long\def\vt#1{\mathbf{#1}}%
\global\long\def\gvt#1{\boldsymbol{#1}}%
\global\long\def\idp{\ \bot\negthickspace\negthickspace\bot\ }%
\global\long\def\cdp{\idp}%
\global\long\def\das{}%
\global\long\def\id{\mathbb{I}}%
\global\long\def\idarg#1#2{\id\left\{  #1,#2\right\}  }%
\global\long\def\iid{\stackrel{\text{iid}}{\sim}}%
\global\long\def\bzero{\vt 0}%
\global\long\def\bone{\mathbf{1}}%
\global\long\def\boldm{\boldsymbol{m}}%
\global\long\def\bff{\vt f}%
\global\long\def\bx{\boldsymbol{x}}%

\global\long\def\bd{\boldsymbol{d}}%
\global\long\def\bl{\boldsymbol{l}}%
\global\long\def\bu{\boldsymbol{u}}%
\global\long\def\bo{\boldsymbol{o}}%
\global\long\def\bh{\boldsymbol{h}}%
\global\long\def\bs{\boldsymbol{s}}%
\global\long\def\bz{\boldsymbol{z}}%
\global\long\def\xnew{y}%
\global\long\def\bxnew{\boldsymbol{y}}%
\global\long\def\bX{\boldsymbol{X}}%
\global\long\def\tbx{\tilde{\bx}}%
\global\long\def\by{\boldsymbol{y}}%
\global\long\def\bY{\boldsymbol{Y}}%
\global\long\def\bZ{\boldsymbol{Z}}%
\global\long\def\bU{\boldsymbol{U}}%
\global\long\def\bv{\boldsymbol{v}}%
\global\long\def\bn{\boldsymbol{n}}%
\global\long\def\bV{\boldsymbol{V}}%
\global\long\def\bI{\boldsymbol{I}}%
\global\long\def\bw{\vt w}%
\global\long\def\balpha{\gvt{\alpha}}%
\global\long\def\bbeta{\gvt{\beta}}%
\global\long\def\bmu{\gvt{\mu}}%
\global\long\def\btheta{\boldsymbol{\theta}}%
\global\long\def\blambda{\boldsymbol{\lambda}}%
\global\long\def\bgamma{\boldsymbol{\gamma}}%
\global\long\def\bpsi{\boldsymbol{\psi}}%
\global\long\def\bphi{\boldsymbol{\phi}}%
\global\long\def\bpi{\boldsymbol{\pi}}%
\global\long\def\bomega{\boldsymbol{\omega}}%
\global\long\def\bepsilon{\boldsymbol{\epsilon}}%
\global\long\def\btau{\boldsymbol{\tau}}%
\global\long\def\bxi{\boldsymbol{\xi}}%
\global\long\def\realset{\mathbb{R}}%
\global\long\def\realn{\realset^{n}}%
\global\long\def\integerset{\mathbb{Z}}%
\global\long\def\natset{\integerset}%
\global\long\def\integer{\integerset}%

\global\long\def\natn{\natset^{n}}%
\global\long\def\rational{\mathbb{Q}}%
\global\long\def\rationaln{\rational^{n}}%
\global\long\def\complexset{\mathbb{C}}%
\global\long\def\comp{\complexset}%

\global\long\def\compl#1{#1^{\text{c}}}%
\global\long\def\and{\cap}%
\global\long\def\compn{\comp^{n}}%
\global\long\def\comb#1#2{\left({#1\atop #2}\right) }%
\global\long\def\nchoosek#1#2{\left({#1\atop #2}\right)}%
\global\long\def\param{\vt w}%
\global\long\def\Param{\Theta}%
\global\long\def\meanparam{\gvt{\mu}}%
\global\long\def\Meanparam{\mathcal{M}}%
\global\long\def\meanmap{\mathbf{m}}%
\global\long\def\logpart{A}%
\global\long\def\simplex{\Delta}%
\global\long\def\simplexn{\simplex^{n}}%
\global\long\def\dirproc{\text{DP}}%
\global\long\def\ggproc{\text{GG}}%
\global\long\def\DP{\text{DP}}%
\global\long\def\ndp{\text{nDP}}%
\global\long\def\hdp{\text{HDP}}%
\global\long\def\gempdf{\text{GEM}}%
\global\long\def\rfs{\text{RFS}}%
\global\long\def\bernrfs{\text{BernoulliRFS}}%
\global\long\def\poissrfs{\text{PoissonRFS}}%
\global\long\def\grad{\gradient}%
\global\long\def\gradient{\nabla}%
\global\long\def\partdev#1#2{\partialdev{#1}{#2}}%
\global\long\def\partialdev#1#2{\frac{\partial#1}{\partial#2}}%
\global\long\def\partddev#1#2{\partialdevdev{#1}{#2}}%
\global\long\def\partialdevdev#1#2{\frac{\partial^{2}#1}{\partial#2\partial#2^{\top}}}%
\global\long\def\closure{\text{cl}}%
\global\long\def\cpr#1#2{\Pr\left(#1\ |\ #2\right)}%
\global\long\def\var{\text{Var}}%
\global\long\def\Var#1{\text{Var}\left[#1\right]}%
\global\long\def\cov{\text{Cov}}%
\global\long\def\Cov#1{\cov\left[ #1 \right]}%
\global\long\def\COV#1#2{\underset{#2}{\cov}\left[ #1 \right]}%
\global\long\def\corr{\text{Corr}}%
\global\long\def\sst{\text{T}}%
\global\long\def\SST{\sst}%
\global\long\def\ess{\mathbb{E}}%

\global\long\def\Ess#1{\ess\left[#1\right]}%
\newcommandx\ESS[2][usedefault, addprefix=\global, 1=]{\underset{#2}{\ess}\left[#1\right]}%
\global\long\def\fisher{\mathcal{F}}%

\global\long\def\bfield{\mathcal{B}}%
\global\long\def\borel{\mathcal{B}}%
\global\long\def\bernpdf{\text{Bernoulli}}%
\global\long\def\betapdf{\text{Beta}}%
\global\long\def\dirpdf{\text{Dir}}%
\global\long\def\gammapdf{\text{Gamma}}%
\global\long\def\gaussden#1#2{\text{Normal}\left(#1, #2 \right) }%
\global\long\def\gauss{\mathbf{N}}%
\global\long\def\gausspdf#1#2#3{\text{Normal}\left( #1 \lcabra{#2, #3}\right) }%
\global\long\def\multpdf{\text{Mult}}%
\global\long\def\poiss{\text{Pois}}%
\global\long\def\poissonpdf{\text{Poisson}}%
\global\long\def\pgpdf{\text{PG}}%
\global\long\def\wshpdf{\text{Wish}}%
\global\long\def\iwshpdf{\text{InvWish}}%
\global\long\def\nwpdf{\text{NW}}%
\global\long\def\niwpdf{\text{NIW}}%
\global\long\def\studentpdf{\text{Student}}%
\global\long\def\unipdf{\text{Uni}}%
\global\long\def\transp#1{\transpose{#1}}%
\global\long\def\transpose#1{#1^{\mathsf{T}}}%
\global\long\def\mgt{\succ}%
\global\long\def\mge{\succeq}%
\global\long\def\idenmat{\mathbf{I}}%
\global\long\def\trace{\mathrm{tr}}%
\global\long\def\argmax#1{\underset{_{#1}}{\text{argmax}} }%
\global\long\def\argmin#1{\underset{_{#1}}{\text{argmin}\ } }%
\global\long\def\diag{\text{diag}}%
\global\long\def\norm{}%
\global\long\def\spn{\text{span}}%
\global\long\def\vtspace{\mathcal{V}}%
\global\long\def\field{\mathcal{F}}%
\global\long\def\ffield{\mathcal{F}}%
\global\long\def\inner#1#2{\left\langle #1,#2\right\rangle }%
\global\long\def\iprod#1#2{\inner{#1}{#2}}%
\global\long\def\dprod#1#2{#1 \cdot#2}%
\global\long\def\norm#1{\left\Vert #1\right\Vert }%
\global\long\def\entro{\mathbb{H}}%
\global\long\def\entropy{\mathbb{H}}%
\global\long\def\Entro#1{\entro\left[#1\right]}%
\global\long\def\Entropy#1{\Entro{#1}}%
\global\long\def\mutinfo{\mathbb{I}}%
\global\long\def\relH{\mathit{D}}%
\global\long\def\reldiv#1#2{\relH\left(#1||#2\right)}%
\global\long\def\KL{KL}%
\global\long\def\KLdiv#1#2{\KL\left(#1\parallel#2\right)}%
\global\long\def\KLdivergence#1#2{\KL\left(#1\ \parallel\ #2\right)}%
\global\long\def\crossH{\mathcal{C}}%
\global\long\def\crossentropy{\mathcal{C}}%
\global\long\def\crossHxy#1#2{\crossentropy\left(#1\parallel#2\right)}%
\global\long\def\breg{\text{BD}}%
\global\long\def\lcabra#1{\left|#1\right.}%
\global\long\def\lbra#1{\lcabra{#1}}%
\global\long\def\rcabra#1{\left.#1\right|}%
\global\long\def\rbra#1{\rcabra{#1}}%

\section{Introduction}

Although annotated data has been shown to be really precious and valuable
to deep learning (DL), in many real-world applications, annotating
a sufficient amount of data for training qualified DL models is prohibitively
labour-expensive, time-consuming, and error-prone. Transfer learning
is a vital solution for the lack of labeled data. Additionally, in
many situations, we are only able to collect limited number of annotated
data from multiple source domains. Therefore, it is desirable to train
a qualified DL model primarily based on multiple source domains and
possibly together with a target domain. Depending on the availability
of the target domain during training, we encounter either \emph{multiple
source domain adaptation} (MSDA) or \emph{domain generalization} (DG)
problem.

Domain adaptation (DA) is a specific case of MSDA when we need to
transfer from a single source domain to an another target domain available
during training. For the DA setting, the pioneering work \citep{DA_ben_david_2010_DAbound}
and other following work \citep{DA_redko_2017_theory_ot_for_da,DA_pmlr-zhang19_Jordan_bridging,DA_theo_zhao19a_on_learning_tradeoff,DA_johansson19a_support_invertibility,MDA_hoffman_nips2018_learn_ens,MSDA_Zhao_NEURIPS2018_adversarial_msda,le2021lamda}
are abundant to study its insightful characterizations and aspects,
notably what \emph{domain-invariant (}\textbf{\emph{DI}}\emph{) representations}
are, how to learn this kind of representations, and the trade-off
of enforcing learning DI representations. Those well-grounded studies
lay the foundation for developing impressive practical DA methods
\citep{DA_ganin_DANN_2016,DA_tzeng_cvpr2017_DANN_variants,DA_Saito_NEURIPS2020,DA_pmlr-zhang19_Jordan_bridging,DA_Cortes_JMLR_2019_gen_disc,DA_Yan_2020_mixup_train,tuan2021tidot}.

Due to the appearance of multiple source domains and possibly unavailability
of target domain, establishing theoretical foundation and characterizing
\emph{DI representations} for the MSDA and especially DG settings
are significantly more challenging. A number of state-of-the-art works
in MSDA and DG \citep{Zhao_Wang_Zhang_Gu_Li_Song_Xu_Hu_Chai_Keutzer_2020,MSDA_Peng_iccv2019_moment_matching,DG_huang2_eccv20_self_challenging,DG_dou_nips2019_model_agnostic_learning,DG_piratla_icml20a_pmlr_low_rank_decom,DG_Li_MMD_DANN_2018_CVPR,DG_Li_CDANN_2018_ECCV,DG_muandet_icml13_kernel,DG_ghifary_iccv2015,DG_Zhao_neurips2020_entropy_max,DG_Haohan_iclr_2019_Projecting_Superficial_Statistics,tuan2021most,nguyen2021stem}
have implicitly exploited and used DI representations in some sense
to achieve impressive performances, e.g., \citep{DG_Li_CDANN_2018_ECCV,DG_Li_CDANN_2018_ECCV,DG_dou_nips2019_model_agnostic_learning}
minimize representation divergence during training, \citep{DG_piratla_icml20a_pmlr_low_rank_decom}
decomposes representation to obtain invariant feature, \citep{DG_Zhao_neurips2020_entropy_max}
maximizes domain prediction entropy to learn invariance. Despite the
successes, the notion DI representations in MSDA and DG is not fully-understood,
and rigorous studies to theoretically characterize DI representations
for these settings are still very crucial and imminent. In this paper,
we provide theoretical answers to the following questions: (1) what
are DI representations in MSDA and DG, and (2) what one should expect
when learning them. Overall, our contributions in this work can be
summarized as: 
\begin{itemize}
\item In Section \ref{subsec:feature-definition}, we first develop two
upper-bounds for the general target loss in the MSDA and DG settings,
whose proofs can be found in Appendix A. We then base on these bounds
to characterize and define two kinds of DI representations: \emph{general
DI representations} and \emph{compressed DI representations}. 
\item We further develop theory to inspect the pros and cons of two kinds
of DI representations in Section \ref{subsec:Feature-Characteristics},
aiming to shed light on how to use those representations in practice.
Particularly, two types of DI representations optimize different types
of divergence on feature space, hence serving different purposes.
Appendix B contains proofs regarding these characteristics. 
\item Finally, we study in Section \ref{subsec:Trade-off} the trade-off
of two kinds of DI representations which theoretically answers the
question whether and how the target performance is hurt when enforcing
learning each kind of representation. We refer reader to Appendix
C for its proof. 
\item We conduct experiments to investigate the trade-off of two kinds of
representations for giving practical hints regarding how to use those
representations in practice as well as exploring other interesting
properties of our developed theory. 
\end{itemize}
% To the best of our knowledge, ours is the first work that rigorously formulates and provides theoretical analysis for DI representations in MSDA and DG.
It is worth noting that although MSDA has been investigated in some
works \citep{MDA_hoffman_nips2018_learn_ens,MSDA_Zhao_NEURIPS2018_adversarial_msda},
our proposed method is the first work that rigorously formulates and
provides theoretical analysis for representation learning in MSDA
and DG. Specifically, our bounds developed in Theorems \ref{thm:general_dom_inv}
and \ref{thm:compact_dom_inv} interweaving both input and latent
spaces are novel and benefit theoretical analyses in deep learning.
Our theory is developed in a general setting of multi-class classification
with a sufficiently general loss, which can be viewed as a non-trivial
generalization of existing works in DA \citep{DA_mansour_colt09_DAbound_genloss,DA_ben_david_2010_DAbound,DA_theo_zhao19a_on_learning_tradeoff,MDA_mansour_nips2009_ens_hypo,MSDA_Zhao_NEURIPS2018_adversarial_msda},
each of which considers binary classification with specific loss functions.
Moreover, our theoretical bounds developed in Theorem \ref{thm:trade_off_LowerBound}
is the first theoretical analysis of the trade-off of learning DI
representations in MSDA and DG. Particularly, by considering the MSDA
setting with single source and target domains, we achieve the same
trade-off nature discovered in \citep{DA_theo_zhao19a_on_learning_tradeoff},
but again our setting is more general than binary classification setting
in that work.

\section{Our Main Theory\label{sec:Main-Theory}}

\subsection{Notations}

Let $\mathcal{X}$ be a data space on which we endow a data distribution
$\mathbb{P}$ with a corresponding density function $p(x)$. We consider
the multi-class classification problem with the label set $\mathcal{Y}=\left[C\right]$,
where $C$ is the number of classes and $\left[C\right]:=\{1,\ldots,C\}$.
Denote $\mathcal{Y}_{\Delta}:=\left\{ \alpha\in\mathbb{R}^{C}:\norm{\alpha}_{1}=1\,\wedge\,\alpha\geq\bzero\right\} $
as the $C-$simplex label space, let $f:\mathcal{X}\mapsto\mathcal{Y}_{\Delta}$
be a probabilistic labeling function returning a $C$-tuple $f\left(x\right)=\left[f\left(x,i\right)\right]_{i=1}^{C}$,
whose element $f\left(x,i\right)=p\left(y=i\mid x\right)$ is the
probability to assign a data sample $x\sim\mathbb{P}$ to the class
$i$ (i.e., $i\in\left\{ 1,...,C\right\} $). Moreover, a domain is
denoted compactly as pair of data distribution and labeling function
$\mathbb{D}:=\left(\mathbb{P},f\right)$. We note that given a data
sample $x\sim\mathbb{P}$, its categorical label $y\in\mathcal{Y}$
is sampled as $y\sim Cat\left(f\left(x\right)\right)$ which a categorical
distribution over $f\left(x\right)\in\mathcal{Y}_{\Delta}$.

Let $l:\mathcal{Y}_{\Delta}\times\mathcal{Y}\mapsto\mathbb{R}$ be
a loss function, where $l\left(\hat{f}\left(x\right),y\right)$ with
$\hat{f}\left(x\right)\in\mathcal{Y}_{\simplex}$ and $y\in\mathcal{Y}$
specifies the loss (e.g., cross-entropy, Hinge, L1, or L2 loss) to
assign a data sample $x$ to the class $y$ by the hypothesis $\hat{f}$.
Moreover, given a prediction probability $\hat{f}\left(x\right)$
w.r.t. the ground-truth prediction $f\left(x\right)$, we define the
loss $\ell\left(\hat{f}\left(x\right),f\left(x\right)\right)=\mathbb{E}_{y\sim f\left(x\right)}\left[l\left(\hat{f}\left(x\right),y\right)\right]=\sum_{y=1}^{C}l\left(\hat{f}\left(x\right),y\right)f\left(x,y\right)$.
This means $\ell\left(\cdot,\cdot\right)$ is convex w.r.t. the second
argument. We further define the general loss caused by using a classifier
$\hat{f}:\mathcal{X}\mapsto\mathcal{Y}_{\Delta}$ to predict $\mathbb{D}\equiv\left(\mathbb{P},f\right)$
as 
\[
\mathcal{L}\left(\hat{f},f,\mathbb{P}\right)=\mathcal{L}\left(\hat{f},\mathbb{D}\right):=\mathbb{E}_{x\sim\mathbb{P}}\left[\ell\left(\hat{f}(x),f(x)\right)\right].
\]
We inspect the multiple source setting in which we are given multiple
source domains $\{\mathbb{D}^{S,i}\}_{i=1}^{K}$ over the common data
space $\mathcal{X}$, each of which consists of data distribution
and its own labeling function $\mathbb{D}^{S,i}:=\left(\mathbb{P}^{S,i},f^{S,i}\right)$.
Based on these source domains, we need to work out a learner or classifier
that requires to evaluate on a target domain $\mathbb{D}^{T}:=\left(\mathbb{P}^{T},f^{T}\right)$.
Depending on the \emph{knownness} or \emph{unknownness} of a target
domain during training, we experience either \emph{multiple source
DA} (MSDA) \citep{MDA_mansour_nips2009_ens_hypo,MDA_mansour_MDA_limited_target_pmlr-21,MDA_redko_aistat19_ot_for_MDA,MDA_hoffman_nips2018_learn_ens}
or \emph{domain generalization} (DG) \citep{DG_DaLi_MLDG_AAAI18,DG_Li_CDANN_2018_ECCV,DG_Li_MMD_DANN_2018_CVPR,DG_muandet_icml13_kernel}
setting.

One typical approach in MSDA and DG is to combine the source domains
together \citep{DA_ganin_DANN_2016,DG_Li_CDANN_2018_ECCV,DG_Li_MMD_DANN_2018_CVPR,DG_muandet_icml13_kernel,DG_albuquerque_2020_dist_matching,DG_dou_nips2019_model_agnostic_learning,DG_DGDANN_sicilia_2021}
to learn \emph{DI representations} with hope to generalize well to
a target domain. When combining the source domains, we obtain a mixture
of multiple source distributions denoted as $\mathbb{D}^{\pi}=\sum_{i=1}^{K}\pi_{i}\mathbb{D}^{S,i}$,
where the mixing coefficients $\pi=\left[\pi_{i}\right]_{i=1}^{K}$
can be conveniently set to $\pi_{i}=\frac{N_{i}}{\sum_{j=1}^{K}N_{j}}$
with $N_{i}$ being the training size of the $i-$th source domain.

In term of representation learning, input is mapped into a latent
space $\mathcal{Z}$ by a feature map $g:\mathcal{X}\mapsto\mathcal{Z}$
and then a classifier $\hat{h}:\mathcal{Z}\mapsto\mathcal{Y}_{\Delta}$
is trained based on the representations $g\left(\mathcal{X}\right)$.
Let $f:\mathcal{X}\mapsto\mathcal{Y}_{\Delta}$ be the original labeling
function. To facilitate the theory developed for latent space, we
introduce representation distribution being the pushed-forward distribution
$\mathbb{P}_{g}:=g_{\#}\mathbb{P}$, and the labeling function $h:\mathcal{Z}\mapsto\mathcal{Y}_{\Delta}$
induced by $g$ as $h(z)=\frac{\int_{g^{-1}\left(z\right)}f(x)p(x)dx}{\int_{g^{-1}\left(z\right)}p(x)dx}$
\citep{DA_johansson19a_support_invertibility}. Going back to our
multiple source setting, the source mixture becomes $\mathbb{D}_{g}^{\pi}=\sum_{i}\pi_{i}\mathbb{D}_{g}^{S,i}$,
where each source domain is $\mathbb{D}_{g}^{S,i}=\left(\mathbb{P}_{g}^{S,i},h^{S,i}\right)$,
and the target domain is $\mathbb{D}_{g}^{T}=\left(\mathbb{P}_{g}^{T},h^{T}\right)$.

Finally, in our theory development, we use Hellinger divergence between
two distributions defined as $\begin{aligned}D_{1/2}\left(\mathbb{P}^{1},\mathbb{P}^{2}\right)=2\int\left(\sqrt{p^{1}(x)}-\sqrt{p^{2}(x)}\right)^{2}dx\end{aligned}
$, whose squared $d_{1/2}=\sqrt{D_{1/2}}$ is a proper metric.

\subsection{Two Types of Domain-Invariant Representations\label{subsec:feature-definition}}

Hinted by the theoretical bounds developed in \citep{DA_ben_david_2010_DAbound},
\emph{DI representations learning}, in which feature extractor $g$
maps source and target data distributions to a common distribution
on the latent space, is well-grounded for the DA setting. However,
this task becomes significantly challenging for the MSDA and DG settings
due to multiple source domains and potential unknownness of target
domain. Similar to the case of DA, it is desirable to develop theoretical
bounds that directly motivate definitions of DI representations for
the MSDA and DG settings, as being done in the next theorem. 
\begin{thm}
\label{thm:general_dom_inv}Consider a mixture of source domains $\mathbb{D}^{\pi}=\sum_{i=1}^{K}\pi_{i}\mathbb{D}^{S,i}$
and the target domain $\mathbb{D}^{T}$. Let $\ell$ be any loss function
upper-bounded by a positive constant $L$. For any hypothesis $\hat{f}:\mathcal{X}\mapsto\mathcal{Y}_{\Delta}$
where $\hat{f}=\hat{h}\circ g$ with $g:\mathcal{X}\goto\mathcal{Z}$
and $\hat{h}:\mathcal{Z}\goto\mathcal{Y}_{\simplex}$, the target
loss on input space is upper bounded 
\begin{equation}
\begin{aligned}\mathcal{L}\left(\hat{f},\mathbb{D}^{T}\right)\leq\sum_{i=1}^{K}\pi_{i}\mathcal{L}\left(\hat{f},\mathbb{D}^{S,i}\right)+L\max_{i\in[K]}\mathbb{E}_{\mathbb{P}^{S,i}}\left[\|\Delta p^{i}(y|x)\|_{1}\right]+L\sqrt{2}\,d_{1/2}\left(\mathbb{P}_{g}^{T},\mathbb{P}_{g}^{\pi}\right)%\mathcal{L}\left(\hat{f},\mathbb{D}^{T}\right)\leq\sum_{i=1}^{K}\pi_{i}\mathcal{L}\left(\hat{f},\mathbb{D}^{S,i}\right)
%+L\sqrt{2}\,d_{1/2}\left(\mathbb{P}_{g}^{T},\mathbb{P}_{g}^{\pi}\right)
%+\max_{i\in[K]}LS\left(\mathbb{D}^{T},\mathbb{D}^{S,i}\right)
%+\max_{i\in[K]}DS\left(\mathbb{D}^{T},\mathbb{D}^{S,i}\right)
\end{aligned}
\label{eq:general_DI_UpperBound}
\end{equation}
where $\Delta p^{i}(y|x):=\left[\left|f^{T}(x,y)-f^{S,i}(x,y)\right|\right]_{y=1}^{C}$
is the absolute of single point label shift on input space between
source domain $\mathbb{D}^{S,i}$, the target domain $\mathbb{D}^{T}$,
and $[K]:=\left\{ 1,2,...,K\right\} $. 
\end{thm}

The bound in Equation \ref{eq:general_DI_UpperBound} implies that
the target loss in the input or latent space depends on three terms:
(i) \emph{representation discrepancy}: $d_{1/2}\left(\mathbb{P}_{g}^{T},\mathbb{P}_{g}^{\pi}\right)$,
(ii) \emph{the label shift}: $\max_{i\in[K]}\mathbb{E}_{\mathbb{P}^{S,i}}\left[\|\Delta p^{i}(y|x)\|_{1}\right]$,
and (iii) \emph{the general source loss}: $\sum_{i=1}^{K}\pi_{i}\mathcal{L}\left(\hat{f},\mathbb{D}^{S,i}\right)$.
To minimize the target loss in the left side, we need to minimize
the three aforementioned terms. First, the \emph{label shift} term
is a natural characteristics of domains, hence almost impossible to
tackle. Secondly, the \emph{representation discrepancy} term can be
explicitly tackled for the MSDA setting, while almost impossible for
the DG setting. Finally, the\emph{ general source loss} term is convenient
to tackle, where its minimization results in a feature extractor $g$
and a classifier $\hat{h}$.

Contrary to previous works in DA and MSDA \citep{DA_ben_david_2010_DAbound,MDA_mansour_nips2009_ens_hypo,MSDA_Zhao_NEURIPS2018_adversarial_msda,DA_Cortes_JMLR_2019_gen_disc}
that consider both losses and data discrepancy on data space, our
bound connects losses on data space to discrepancy on representation
space. Therefore, our theory provides a natural way to analyse representation
learning, especially feature alignment in deep learning practice.
Note that although DANN \citep{DA_ganin_DANN_2016} explains their
feature alignment method using theory developed by Ben-david et al.
\citep{DA_ben_david_2010_DAbound}, it is not rigorous. In particular,
while application of the theory to representation space yield a representation
discrepancy term, the loss terms are also on that feature space, and
hence minimizing these losses is not the learning goal. Finally, our
setting is much more general, which extends to multilabel, stochastic
labeling setting, and any bounded loss function. % Firstly, Ben-david et al. 2010 \cite{DA_ben_david_2010_DAbound}

From the upper bound, we turn to first type of DI representations
for the MSDA and DG settings. Here, the objective of $g$ is to map
samples onto representation space in a way that the common classifier
$\hat{h}$ can effectively and correctly classifies them, regardless
of which domain the data comes from.
\begin{defn}
\textit{\label{def:general_DI_rep}Consider a class $\mathcal{G}$
of feature maps and a class $\mathcal{H}$ of hypotheses. Let $\{\mathbb{D}^{S,i}\}_{i=1}^{K}$
be a set of source domains. }

\textit{i) (}\textbf{\textit{DG with unknown target data}}\textit{)
A feature map $g^{*}\in\mathcal{G}$ is said to be a }\textbf{\textit{DG}}\textit{
general domain-invariant (}\textbf{\textit{DI}}\textit{) feature map
if $g^{*}$ is the solution of the optimization problem (OP): $\min_{g\in\mathcal{G}}\min_{\hat{h}\in\mathcal{H}}\sum_{i=1}^{K}\pi_{i}\mathcal{L}\left(\hat{h},\mathbb{D}_{g}^{S,i}\right)$.
Moreover, the latent representations $z=g^{*}\left(x\right)$ induced
by $g^{*}$ is called general DI representations for the DG setting. }

\emph{ii)} \textit{(}\textbf{\textit{MSDA with known target data}}\textit{)
A feature map $g^{*}\in\mathcal{G}$ is said to be a }\textbf{\textit{MSDA}}\textit{
general DI feature map if $g^{*}$ is the solution of the optimization
problem (OP): $\min_{g\in\mathcal{G}}\min_{\hat{h}\in\mathcal{H}}\sum_{i=1}^{K}\pi_{i}\mathcal{L}\left(\hat{h},\mathbb{D}_{g}^{S,i}\right)$
which satisfies $\mathbb{P}_{g^{*}}^{T}=\mathbb{P}_{g^{*}}^{\pi}$
(i.e., $\min_{g\in\mathcal{G}}d_{1/2}\left(\mathbb{P}_{g}^{T},\mathbb{P}_{g}^{\pi}\right)$).
Moreover, the latent representations $z=g^{*}\left(x\right)$ induced
by $g^{*}$ is called general DI representations for the MSDA setting.} 
\end{defn}

The definition of general DI representations for the DG setting in
Definition \ref{def:general_DI_rep} is transparent in light of Theorem
\ref{thm:general_dom_inv}, wherein we aim to find $g$ and $\hat{h}$
to minimize the general source loss $\sum_{i=1}^{K}\pi_{i}\mathcal{L}\left(\hat{f},\mathbb{D}^{S,i}\right)$
due to the unknownness of $\mathbb{P}^{T}$. Meanwhile, regarding
the general DI representations for the MSDA setting, we aim to find
$g^{*}$ satisfying $\mathbb{P}_{g^{*}}^{T}=\mathbb{P}_{g^{*}}^{\pi}$
and $\hat{h}$ to minimize the general source loss $\sum_{i=1}^{K}\pi_{i}\mathcal{L}\left(\hat{f},\mathbb{D}^{S,i}\right)$.
Practically, to find general representations for MSDA, we solve 
\[
\min_{g\in\mathcal{G}}\min_{\hat{h}\in\mathcal{H}}\left\{ \sum_{i=1}^{K}\pi_{i}\mathcal{L}\left(\hat{h},\mathbb{D}_{g}^{S,i}\right)+\lambda D\left(\mathbb{P}_{g}^{T},\mathbb{P}_{g}^{\pi}\right)\right\} ,
\]
where $\lambda>0$ is a trade-off parameter and $D$ can be any divergence
(e.g., Jensen Shannon (JS) divergence \citep{GAN_Goodfellow}, $f$-divergence
\citep{xuanlong_f_div_NIPS2007}, MMD distance \citep{DG_Li_MMD_DANN_2018_CVPR},
or WS distance \citep{MDA_redko_aistat19_ot_for_MDA}).

In addition, the \textit{\emph{general DI representations}} have been
exploited in some works \citep{DG_zhou_eccv2020_generate_domains,DG_Li_2019_episodic_AGG}
from the practical perspective and previously discussed in \citep{OOD_IRM_arjovsky_2020}
from the theoretical perspective. Despite the similar definition to
our work, general DI representation in \citep{OOD_IRM_arjovsky_2020}
is not motivated from minimization of a target loss bound.

As pointed out in the next section, learning general DI representations
increases the span of latent representations $g\left(x\right)$ on
the latent space which might help to reduce $d_{1/2}\left(\mathbb{P}_{g}^{T},\mathbb{P}_{g}^{\pi}\right)$
for a general target domain. On the other hand, many other works \citep{DG_Li_MMD_DANN_2018_CVPR,DG_Li_CDANN_2018_ECCV,DG_ghifary_iccv2015,DG_Haohan_iclr_2019_Projecting_Superficial_Statistics,DG_dou_nips2019_model_agnostic_learning}
have also explored the possibility of enhancing generalization by
finding common representation among source distributions. The following
theorem motivates the latter, which is the second kind of DI representations. 
\begin{thm}
\label{thm:compact_dom_inv}Consider a mixture of source domains $\mathbb{D}^{\pi}=\sum_{i=1}^{K}\pi_{i}\mathbb{D}^{S,i}$
and the target domain $\mathbb{D}^{T}$. Let $\ell$ be any loss function
upper-bounded by a positive constant $L$. For any hypothesis $\hat{f}:\mathcal{X}\mapsto\mathcal{Y}_{\Delta}$
where $\hat{f}=\hat{h}\circ g$ with $g:\mathcal{X}\goto\mathcal{Z}$
and $\hat{h}:\mathcal{Z}\goto\mathcal{Y}_{\simplex}$, the target
loss on input space is upper bounded{\footnotesize{}{} 
\begin{equation}
\begin{aligned}\mathcal{L}\left(\hat{f},\mathbb{D}^{T}\right) & \leq\sum_{i=1}^{K}\pi_{i}\mathcal{L}\left(\hat{f},\mathbb{D}^{S,i}\right)+L\max_{i\in[K]}\mathbb{E}_{\mathbb{P}^{S,i}}\left[\|\Delta p^{i}(y|x)\|_{1}\right]\\
 & +\sum_{i=1}^{K}\sum_{j=1}^{K}\frac{L\sqrt{2\pi_{j}}}{K}d_{1/2}\left(\mathbb{P}_{g}^{T},\mathbb{P}_{g}^{S,i}\right)+\sum_{i=1}^{K}\sum_{j=1}^{K}\,\frac{L\sqrt{2\pi_{j}}}{K}d_{1/2}\left(\mathbb{P}_{g}^{S,i},\mathbb{P}_{g}^{S,j}\right).
\end{aligned}
\label{eq:compact_DI_UpperBound}
\end{equation}
}{\footnotesize\par}

\end{thm}

Evidently, Theorem \ref{thm:compact_dom_inv} suggests another kind
of representation learning, where source representation distributions
are aligned in order to lower the target loss's bound as concretely
defined as follows. 
\begin{defn}
\label{def:compress_DI_rep}Consider a class $\mathcal{G}$ of feature
maps and a class $\mathcal{H}$ of hypotheses. Let $\left\{ \mathbb{D}^{S,i}\right\} _{i=1}^{K}$
be a set of source domains.

\textit{i) (}\textbf{\textit{DG with unknown target data}}\textit{)
A feature map $g^{*}\in\mathcal{G}$ is a }\textbf{\textit{DG}}\textit{
compressed DI representations for source domains $\{\mathbb{D}^{S,i}\}_{i=1}^{K}$
if $g^{*}$ is the solution of the optimization problem (OP): $\min_{g\in\mathcal{G}}\text{\ensuremath{\min}}_{\hat{h}\in\mathcal{H}}\sum_{i=1}^{K}\pi_{i}\mathcal{L}\left(\hat{h},\mathbb{D}_{g}^{S,i}\right)$
which satisfies $\mathbb{P}_{g^{*}}^{S,1}=\mathbb{P}_{g^{*}}^{S,2}=\ldots=\mathbb{P}_{g^{*}}^{S,K}$
(i.e., the pushed forward distributions of all }\textbf{\textit{source}}\textit{
domains are identical). The latent representations $z=g^{*}\left(x\right)$
is then called compressed DI representations for the DG setting.}

\textit{ii) (}\textbf{\textit{MSDA with known target data}}\textit{)
A feature map $g^{*}\in\mathcal{G}$ is an }\textbf{\textit{MSDA}}\textit{
compressed DI representations for source domains $\{\mathbb{D}^{S,i}\}_{i=1}^{K}$
if $g^{*}$ is the solution of the optimization problem (OP): $\min_{g\in\mathcal{G}}\text{\ensuremath{\min}}_{\hat{h}\in\mathcal{H}}\sum_{i=1}^{K}\pi_{i}\mathcal{L}\left(\hat{h},\mathbb{D}_{g}^{S,i}\right)$
which satisfies $\mathbb{P}_{g^{*}}^{S,1}=\mathbb{P}_{g^{*}}^{S,2}=\ldots=\mathbb{P}_{g^{*}}^{S,K}=\mathbb{P}_{g^{*}}^{T}$
(i.e., the pushed forward distributions of all }\textbf{\textit{source}}\textit{
and }\textbf{\textit{target}}\textit{ domains are identical). The
latent representations $z=g^{*}\left(x\right)$ is then called compressed
DI representations for the MSDA setting.} 
\end{defn}

\subsection{Characteristics of Domain-Invariant Representations\label{subsec:Feature-Characteristics}}

In the previous section, two kinds of DI representations are introduced,
where each of them originates from minimization of different terms
in the upper bound of target loss. In what follows, we examine and
discuss their benefits and drawbacks. We start with the novel development
of \emph{hypothesis-aware divergence for multiple distributions },
which is necessary for our theory developed later.

\subsubsection{Hypothesis-Aware Divergence for Multiple Distributions \label{subsec:multi_div}}

Let consider multiple distributions $\mathbb{Q}_{1},...,\mathbb{Q}_{C}$
on the same space, whose density functions are $q_{1},...,q_{C}$.
We are given a mixture of these distributions $\mathbb{Q}^{\alpha}=\sum_{i=1}^{C}\alpha_{i}\mathbb{Q}_{i}$
with a mixing coefficient $\alpha\in\mathcal{Y}_{\simplex}$ and desire
to measure a divergence between $\mathbb{Q}_{1},...,\mathbb{Q}_{C}$.
One possible solution is employing a hypothesis from an infinite capacity
hypothesis class $\mathcal{H}$ to identify which distribution the
data come from. In particular, given $z\sim\mathbb{Q}^{\alpha}$,
the function $\hat{h}\left(z,i\right),i\in[C]$ outputs the probability
that $x\sim\mathbb{Q}_{i}$. Note that we use the notation $\hat{h}$
in this particular subsection to denote a domain classifier, while
everywhere else in the paper $\hat{h}$ denotes label classifier.
Let $l\left(\hat{h}\left(z\right),i\right)\in\mathbb{R}$ be the loss
when classifying $z$ using $\hat{h}$ provided the ground-truth label
$i\in\left[C\right]$. Our motivation is that if the distributions
$\mathbb{Q}_{1},...,\mathbb{Q}_{C}$ are distant, it is easier to
distinguish samples from them, hence the minimum classification loss
$\min_{\hat{h}\in\mathcal{H}}\mathcal{L}_{\mathbb{Q}_{1:C}}^{\alpha}\left(\hat{h}\right)$
is much lower than in the case of clutching $\mathbb{Q}_{1},...,\mathbb{Q}_{C}$.
Therefore, we develop the following theorem to connect the loss optimization
problem with an f-divergence among $\mathbb{Q}_{1},...,\mathbb{Q}_{C}$.
More discussions about hypothesis-aware divergence can be found in
Appendix B of this paper and in \citep{duchi2017_clf_and_divergence}. 
\begin{thm}
\label{thm:op_distance} Assuming the hypothesis class $\mathcal{H}$
has infinite capacity, we define the hypothesis-aware divergence for
multiple distributions as 
\begin{equation}
D^{\alpha}\left(\mathbb{Q}_{1},...,\mathbb{Q}_{C}\right)=-\min_{\hat{h}\in\mathcal{H}}\mathcal{L}_{\mathbb{Q}_{1:C}}^{\alpha}\left(\hat{h}\right)+\mathcal{C}_{l,\alpha},\label{eq:loss_distance_connection}
\end{equation}
where $\mathcal{C}_{l,\alpha}$ depends only on the form of loss function
$l$ and value of $\alpha$. This divergence is a proper f-divergence
among $\mathbb{Q}_{1},...,\mathbb{Q}_{C}$ in the sense that $D^{\alpha}\left(\mathbb{Q}_{1},...,\mathbb{Q}_{C}\right)\geq0,\forall\mathbb{Q}_{1},...,\mathbb{Q}_{C}$
and $\alpha\in\mathcal{Y}_{\simplex}$, and $D^{\alpha}\left(\mathbb{Q}_{1},...,\mathbb{Q}_{C}\right)=0$
if $\mathbb{Q}_{1}=...=\mathbb{Q}_{C}$. 
\end{thm}

\subsubsection{General Domain-Invariant Representation \label{subsec:General_features}}

As previously defined, the general DI feature map $g^{*}$ (cf. Definition
\ref{def:general_DI_rep}) is the one that minimizes the total source
loss 
\begin{equation}
\begin{aligned}g^{*}=\arg\min_{g\in\mathcal{G}}\min_{\hat{h}\in\mathcal{H}}\sum_{i=1}^{K}\pi_{i}\mathcal{L}\left(\hat{h},h^{S,i},\mathbb{P}_{g}^{i}\right)=\text{argmin}_{g\in\mathcal{G}}\text{min}_{\hat{h}\in\mathcal{H}}\sum_{i=1}^{K}\pi_{i}\mathcal{L}\left(\hat{h},\mathbb{D}_{g}^{S,i}\right)\end{aligned}
.\label{eq:optimal_g_general_feature}
\end{equation}
From the result of Theorem \ref{thm:op_distance}, we expect that
the minimal loss $\min_{\hat{h}\in\mathcal{H}}\mathcal{L}\left(\hat{h},h^{S,i},\mathbb{P}_{g}^{S,i}\right)$
should be inversely proportional to the divergence between the class-conditionals
$\mathbb{P}_{g}^{S,i,c}$. To generalize this result to the multi-source
setting, we further define $\mathbb{Q}_{g}^{S,c}:=\sum_{i=1}^{K}\frac{\pi_{i}\gamma_{i,c}}{\alpha_{c}}\mathbb{P}_{g}^{S,i,c}$
as the mixture of the class $c$ conditional distributions of the
source domains on the latent space, where $\alpha_{c}=\sum_{j=1}^{K}\pi_{j}\gamma_{j,c}$
are the mixing coefficients, and $\gamma_{i,c}=\mathbb{P}^{S,i}\left(y=c\right)$
are label marginals. Then, the inner loop of the objective function
in Eq. \ref{eq:optimal_g_general_feature} can be viewed as training
the optimal hypothesis $\hat{h}\in\mathcal{H}$ to distinguish the
samples from $\mathbb{Q}_{g}^{S,c},c\in\left[C\right]$ for a given
feature map $g$. Therefore, by linking to the multi-divergence concept
developed in Section \ref{subsec:multi_div}, we achieve the following
theorem. 
\begin{thm}
\label{thm:max_distance_g}Assume that $\mathcal{H}$ has infinite
capacity, we have the following statements.

1. $D^{\alpha}\left(\mathbb{Q}_{g}^{s,1},...,\mathbb{Q}_{g}^{s,C}\right)=-\min_{\hat{h}\in\mathcal{H}}\sum_{i=1}^{K}\pi_{i}\mathcal{L}\left(\hat{h},h^{S,i},\mathbb{P}_{g}^{S,i}\right)+\text{const}$,
where $\alpha=\left[\alpha_{c}\right]_{c\in\left[C\right]}$ is defined
as above.

2. Finding the general DI feature map $g^{*}$ via the OP in Eq. \ref{eq:optimal_g_general_feature}
is equivalent to solving 
\begin{equation}
g^{*}=\argmax{g\in\mathcal{G}}\,D^{\alpha}\left(\mathbb{Q}_{g}^{s,1},...,\mathbb{Q}_{g}^{s,C}\right).\label{eq:max_distance}
\end{equation}
\end{thm}

Theorem \ref{thm:max_distance_g}, especially its second claim, discloses
that learning general DI representations maximally expands the coverage
of latent representations of the source domains by maximizing $D^{\alpha}\left(\mathbb{Q}_{g}^{s,1},...,\mathbb{Q}_{g}^{s,C}\right)$.
Hence, the span of source mixture $\mathbb{P}_{g}^{\pi}=\sum_{c=1}^{C}\alpha_{c}\mathbb{Q}_{g}^{s,c}$
is also increased. We believe this demonstrates one of the benefits
of general DI representations because it implicitly enhance the chance
to match a general \emph{unseen target domains} in the DG setting
by possibly reducing the source-target representation discrepancy
term $d_{1/2}\left(\mathbb{P}_{g}^{T},\mathbb{P}_{g}^{\pi}\right)$
(cf. Theorem \ref{thm:general_dom_inv}). Additionally, in MSDA where
we wish to minimize $d_{1/2}\left(\mathbb{P}_{g}^{T},\mathbb{P}_{g}^{\pi}\right)$
explicitly, expanding source representation's coverage is also useful.

\subsubsection{Compressed Domain-Invariant Representations}

It is well-known that generalization gap between the true loss and
its empirical estimation affects generalization capability of model
\citep{vapnik1998statistical,shalev2014understanding}. One way to
close this gap is increasing sampling density. However, as hinted
by our analysis, enforcing general DI representation learning tends
to maximize the cross-domain class divergence, hence implicitly increasing
the diversity of latent representations from different source domains.
This renders learning a classifier $\hat{h}$ on top of those source
representations harder due to scattering samples on representation
space, leading to higher generalization gap between empirical and
general losses. In contrast, compressed DI representations help making
the task of learning $\hat{h}$ easier with a lower generalization
gap by decreasing the diversity of latent representations from the
different source domains, via enforcing $\mathbb{P}_{g^{*}}^{S,1}=\mathbb{P}_{g^{*}}^{S,2}=\ldots=\mathbb{P}_{g^{*}}^{S,K}$.
We now develop rigorous theory to examine this observation.

Let the training set be $S=\{(z_{i},y_{i})\}_{i=1}^{N}$, consisting
of independent pairs $(z_{i},y_{i})$ drawn from the source mixture,
i.e., the sampling process starts with sampling domain index $k\sim Cat\left(\boldsymbol{\pi}\right)$,
then sampling $z\sim\mathbb{P}_{g}^{S,k}$ (i.e., $x\sim\mathbb{P}^{S,k}$
and $z=g\left(x\right)$), and finally assigning a label with $y\sim h^{S,k}\left(z\right)$
(i.e., $h^{S,k}\left(z\right)$ is a distribution over $\left[C\right]$).
The empirical loss is defined as 
\begin{equation}
\begin{aligned}\mathcal{L}\left(\hat{h},S\right)=\frac{1}{\sum_{k=1}^{K}N_{k}}\sum_{k=1}^{K}\sum_{i=1}^{N_{k}}\ell\left(\hat{h}\left(z_{i}\right),h^{S,k}\left(z_{i}\right)\right),\end{aligned}
\label{eq:empirical_loss}
\end{equation}
where $N_{k},\,k\in\left[C\right]$ is the number of samples drawn
from the $k$-th source domain and $N=\sum_{k=1}^{K}N_{k}$.

Here, $\mathcal{L}\left(\hat{h},S\right)$ is an unbiased estimation
of the general loss $\mathcal{L}\left(\hat{h},\mathbb{D}_{g}^{\pi}\right)=\sum_{k=1}^{K}\pi_{k}\mathcal{L}\left(\hat{h},\mathbb{D}_{g}^{S,k}\right)$.
To quantify the quality of the estimation, we investigate the upper
bound of the \emph{generalization gap} $\left|\mathcal{L}\left(\hat{h},S\right)-\mathcal{L}\left(\hat{h},\mathbb{D}_{g}^{\pi}\right)\right|$
with the confidence level $1-\delta$. 
\begin{thm}
\label{thm:gap}For any confident level $\delta\in[0,1]$ over the
choice of $S$, the estimation of loss is in the $\epsilon$-range
of the true loss 
\[
\text{Pr}\left(\left|\mathcal{L}\left(\hat{h},S\right)-\mathcal{L}\left(\hat{h},\mathbb{D}_{g}^{\pi}\right)\right|\leq\epsilon\right)\geq1-\delta,
\]
where $\epsilon=\epsilon\left(\delta\right)=\left(\frac{A}{\delta}\right)^{1/2}$
is a function of $\delta$ for which $A$ is proportional to{\footnotesize{}
\begin{align*}
\begin{aligned}\frac{1}{N}\left(\sum_{i=1}^{K}\sum_{j=1}^{K}\frac{\sqrt{\pi_{i}}}{K}\mathcal{L}\left(\hat{f},\mathbb{D}^{S,j}\right)+L\sum_{i=1}^{K}\sqrt{\pi_{i}}\max_{k\in\left[K\right]}\mathbb{E}_{\mathbb{P}^{S,k}}\left[\left\Vert \Delta p^{k,i}\left(y|x\right)\right\Vert _{1}\right]+\frac{L}{K}\sum_{i=1}^{K}\sum_{j=1}^{K}\sqrt{2\pi_{i}}~d_{1/2}\left(\mathbb{P}_{g}^{S,i},\mathbb{P}_{g}^{S,j}\right)\right)^{2}\end{aligned}
\end{align*}
}{\footnotesize\par}

\end{thm}

Theorem \ref{thm:gap} reveals one benefit of learning compressed
DI representations, that is, when enforcing compressed DI representation
learning, we minimize $\frac{L}{K}\sum_{i=1}^{K}\sum_{j=1}^{K}\sqrt{2\pi_{i}}~d_{1/2}\left(\mathbb{P}_{g}^{S,i},\mathbb{P}_{g}^{S,j}\right)$,
which tends to reduce the generalization gap $\epsilon=\epsilon\left(\delta\right)$
for a given confidence level $1-\delta$. Therefore, compressed DI
representations allow us to minimize population loss $\mathcal{L}\left(\hat{h},\mathbb{D}_{g}^{\pi}\right)$
more efficiently via minimizing empirical loss.

\subsection{Trade-off of Learning Domain Invariant Representation\label{subsec:Trade-off}}

Similar to the theoretical finding in Zhao et al. \citep{DA_theo_zhao19a_on_learning_tradeoff}
developed for DA, we theoretically find that compression does come
with a cost for MSDA and DG. We investigate the representation trade-off,
typically how compressed DI representation affects classification
loss. Specifically, we consider a data processing chain $\mathcal{X}\stackrel{g}{\longmapsto}\mathcal{Z}\stackrel{\hat{h}}{\longmapsto}\mathcal{Y}_{\Delta}$,
where $\mathcal{X}$ is the common data space, $\mathcal{Z}$ is the
latent space induced by a feature extractor $g$, and $\hat{h}$ is
a hypothesis on top of the latent space. We define $\mathbb{P}_{\mathcal{Y}}^{\pi}$
and $\mathbb{P}_{\mathcal{Y}}^{T}$ as two distribution over $\mathcal{Y}$
in which to draw $y\sim\mathbb{P}_{\mathcal{Y}}^{\pi}$, we sample
$k\sim Cat\left(\pi\right)$, $x\sim\mathbb{P}^{S,k}$, and $y\sim f^{S,k}\left(x\right)$,
while similar to draw $y\sim\mathbb{P}_{\mathcal{Y}}^{T}$. Our theoretical
bounds developed regarding the trade-off of learning DI representations
are relevant to $d_{1/2}\left(\mathbb{P}_{\mathcal{Y}}^{\pi},\mathbb{P}_{\mathcal{Y}}^{T}\right)$. 
\begin{thm}
\label{thm:trade_off_LowerBound}Consider a feature extractor $g$
and a hypothesis $\hat{h}$, the Hellinger distance between two label
marginal distributions $\mathbb{P}_{\mathcal{Y}}^{\pi}$ and $\mathbb{P}_{\mathcal{Y}}^{T}$
can be upper-bounded as:

1. $d_{1/2}\left(\mathbb{P}_{\mathcal{Y}}^{\pi},\mathbb{P}_{\mathcal{Y}}^{T}\right)\leq\left[\sum_{k=1}^{K}\pi_{k}\mathcal{L}\left(\hat{h}\circ g,f^{S,k},\mathbb{P}^{S,k}\right)\right]^{1/2}+d_{1/2}\left(\mathbb{P}_{g}^{T},\mathbb{P}_{g}^{\pi}\right)+\mathcal{L}\left(\hat{h}\circ g,f^{T},\mathbb{P}^{T}\right)^{1/2}$

2. $d_{1/2}\left(\mathbb{P}_{\mathcal{Y}}^{\pi},\mathbb{P}_{\mathcal{Y}}^{T}\right)\leq\left[\sum_{i=1}^{K}\pi_{i}\mathcal{L}\left(\hat{h}\circ g,f^{S,i},\mathbb{P}^{S,i}\right)\right]^{1/2}+\sum_{i=1}^{K}\sum_{j=1}^{K}\frac{\sqrt{\pi_{j}}}{K}d_{1/2}\left(\mathbb{P}_{g}^{S,i},\mathbb{P}_{g}^{S,j}\right)+\sum_{i=1}^{K}\sum_{j=1}^{K}\frac{\sqrt{\pi_{j}}}{K}d_{1/2}\left(\mathbb{P}_{g}^{T},\mathbb{P}_{g}^{S,i}\right)+\mathcal{L}\left(\hat{h}\circ g,f^{T},\mathbb{P}^{T}\right)^{1/2}.$

Here we note that the general loss $\mathcal{L}$ is defined based
on the Hellinger loss $\ell$ which is define as

$\ell(\hat{f}(x),f(x))=D_{1/2}(\hat{f}(x),f(x))=2\sum_{i=1}^{C}\left(\sqrt{\hat{f}\left(x,i\right)}-\sqrt{f\left(x,i\right)}\right)^{2}$
(more discussion can be found in Appendix C). 
\end{thm}

\textit{Remark.} Compared to the trade-off bound in the work of Zhao
et al. \citep{DA_theo_zhao19a_on_learning_tradeoff}, our context
is more general, concerning MSDA and DG problems with multiple source
domains and multi-class probabilistic labeling functions, rather than
single source DA with binary-class and deterministic setting. Moreover,
the Hellinger distance is more universal, in the sense that it does
not depend on the choice of classifier family $\mathcal{H}$ and loss
function $\ell$ as in the case of $\mathcal{H}$-divergence in \citep{DA_theo_zhao19a_on_learning_tradeoff}.

We base on the first inequality of Theorem \ref{thm:trade_off_LowerBound}
to analyze the trade-off of learning general DI representations. The
first term on the left hand side is the source mixture's loss, which
is controllable and tends to be small when enforcing learning general
DI representations. With that in mind, if \textit{two label marginal
distributions }$\mathbb{P}_{\mathcal{Y}}^{\pi}$ and $\mathbb{P}_{\mathcal{Y}}^{T}$
are distant (i.e., $d_{1/2}\left(\mathbb{P}_{\mathcal{Y}}^{\pi},\mathbb{P}_{\mathcal{Y}}^{T}\right)$
is high), the sum $d_{1/2}\left(\mathbb{P}_{g}^{T},\mathbb{P}_{g}^{\pi}\right)+\mathcal{L}\left(\hat{h}\circ g,f^{T},\mathbb{P}^{T}\right)^{1/2}$
tends to be high. This leads to 2 possibilities. The first scenario
is when the representation discrepancy $d_{1/2}\left(\mathbb{P}_{g}^{T},\mathbb{P}_{g}^{\pi}\right)$
has small value, e.g., it is minimized in MSDA setting, or it happens
to be small by pure chance in DG setting. In this case, the lower
bound of target loss $\mathcal{L}\left(\hat{h}\circ g,f^{T},\mathbb{P}^{T}\right)$
is high, possibly hurting model's generalization ability. On the other
hand, if the discrepancy $d_{1/2}\left(\mathbb{P}_{g}^{T},\mathbb{P}_{g}^{\pi}\right)$
is large for some reasons, the lower bound of target loss will be
small, but its upper-bound is higher, as indicated Theorem \ref{thm:general_dom_inv}.

Based on the second inequality of Theorem \ref{thm:trade_off_LowerBound},
we observe that if \textit{two label marginal distributions }$\mathbb{P}_{\mathcal{Y}}^{\pi}$
and $\mathbb{P}_{\mathcal{Y}}^{T}$ are distant while enforcing learning
compressed DI representations (i.e., both source loss and source-source
feature discrepancy $\left[\sum_{i=1}^{K}\pi_{i}\mathcal{L}\left(\hat{h}\circ g,f^{S,i},\mathbb{P}^{S,i}\right)\right]^{1/2}+\sum_{i=1}^{K}\sum_{j=1}^{K}\frac{\sqrt{\pi_{j}}}{K}d_{1/2}\left(\mathbb{P}_{g}^{S,i},\mathbb{P}_{g}^{S,j}\right)$
are low), the sum $\sum_{i=1}^{K}\sum_{j=1}^{K}\frac{\sqrt{\pi_{j}}}{K}d_{1/2}\left(\mathbb{P}_{g}^{T},\mathbb{P}_{g}^{S,i}\right)+\mathcal{L}\left(h\circ g,f^{T},\mathbb{P}^{T}\right)^{1/2}$
is high. For the\emph{ MSDA setting}, the discrepancy $\sum_{i=1}^{K}\sum_{j=1}^{K}\frac{\sqrt{\pi}_{j}}{K}d_{1/2}\left(\mathbb{P}_{g}^{T},\mathbb{P}_{g}^{S,i}\right)$
is trained to get smaller, meaning that the lower bound of target
loss $\mathcal{L}\left(\hat{h}\circ g,f^{T},\mathbb{P}^{T}\right)$
is high, hurting the target performance. Similarly, for the \emph{DG
setting}, if the trained feature extractor $g$ occasionally reduces
$\sum_{i=1}^{K}\sum_{j=1}^{K}\frac{\sqrt{\pi}_{j}}{K}d_{1/2}\left(\mathbb{P}_{g}^{T},\mathbb{P}_{g}^{S,i}\right)$
for some unseen target domain, it certainly increases the target loss
$\mathcal{L}\left(h\circ g,f^{T},\mathbb{P}^{T}\right)$. In contrast,
if for some target domains, the discrepancy $\sum_{i=1}^{K}\sum_{j=1}^{K}\frac{\sqrt{\pi}_{j}}{K}d_{1/2}\left(\mathbb{P}_{g}^{T},\mathbb{P}_{g}^{S,i}\right)$
is high by some reasons, by linking to the upper bound in Theorem
\ref{thm:compact_dom_inv}, the target general loss has a high upper-bound,
hence is possibly high.

This trade-off between representation discrepancy and target loss
suggests a sweet spot for just-right feature alignment. In that case,
the target loss is most likely to be small.

\section{Experiment}

To investigate the trade-off of learning DI representations, we conduct
domain generalization experiments on the \emph{colored MNIST} dataset
(CC BY-SA 3.0) \citep{OOD_IRM_arjovsky_2020,lecun2010mnist}. In particular,
the task is to predict binary label $Y$ of colored input images $X$
generated from binary digit feature $Z_{d}$ and color feature $Z_{c}$.
We refer readers to Appendix D for more information of this dataset.

We conduct $7$ source domains by setting color-label correlation
$\mathbb{P}(Z_{c}=1|Y=1)=\mathbb{P}(Z_{c}=0|Y=0)=\theta^{S,i}$ where
$\theta^{s,i}\sim Uni\left(\left[0.6,1\right]\right)$ for $i=1,...,12$,
while two target domains are created with $\theta^{T,i}\in\text{\ensuremath{\left\{ 0.05,0.7\right\} } }$
for $i=1,2$. Here we note that colored images in the target domain
with $\theta^{T,2}=0.7$ are \emph{more similar }to those in the source
domains, while colored images in the target domain with $\theta^{T,1}=0.05$
are \emph{less similar}.

We wish to study the characteristics and trade-off between two kinds
of DI representations when predicting on various target domains. Specifically,
we apply adversarial learning \citep{GAN_Goodfellow} similar to \citep{DA_ganin_DANN_2016},
in which a min-max game is played between domain discriminator $\hat{h}^{d}$
trying to distinguish the source domain given representation, while
the feature extractor (generator) $g$ tries to fool the domain discriminator.
Simultaneously, a classifier is used to classify label based on the
representation. Let $\mathcal{L}_{gen}$ and $\mathcal{L}_{disc}$
be the label classification and domain discrimination losses, the
training objective becomes: 
\[
\min_{g}\left(\min_{\hat{h}}\mathcal{L}_{gen}+\lambda\max_{\hat{h}^{d}}\mathcal{L}_{disc}\right),
\]
where the source compression strength $\lambda>0$ controls the compression
extent of learned representation. More specifically, general DI representation
is obtained with $\lambda=0$, while larger $\lambda$ leads to more
compressed DI representation. Finally, our implementation is based
on DomainBed \citep{domainbed} repository, and all training details
as well as further MSDA experiment and DG experiment on real datasets
are included in Appendix D.\footnote{Our code could be found at: https://github.com/VinAIResearch/DIRep}

\subsection{Trade-off of Two Kinds of Domain-Invariant Representations}
\begin{center}
\begin{figure}[H]
 \centering{}\includegraphics[height=0.12\paperwidth]{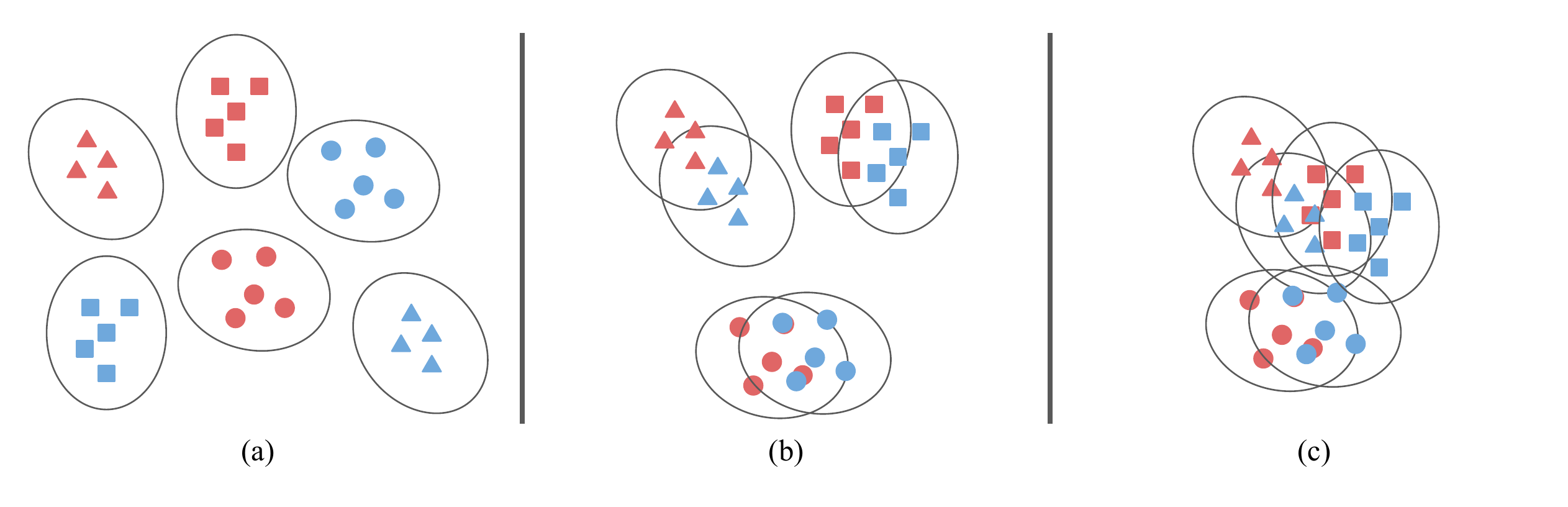}
\caption{\label{fig:feature_illustration}Phase transition of representations
with increased $\lambda$. (a) General DI representations. (b) Just-right
compressed DI representations. (c) Overly-compressed DI representations.}
 
\end{figure}
\par\end{center}

We consider target domain which is similar to source domain in this
experiment. To govern the trade-off between two kinds of DI representations,
we gradually increase the source compression strength $\lambda$ with
noting that $\lambda=0$ means only focusing on learning general DI
representations. According to our theory and as shown in Figure \ref{fig:feature_illustration},
\emph{learning only general DI representations} ($\lambda=0$) maximizes
the cross-domain class divergence by encouraging the classes of source
domains more separate arbitrarily, while by increasing $\lambda$,
we enforce compressing the latent representations of source domains
together. Therefore, for an appropriate $\lambda$, the class separation
from general DI representations and the source compression from compressed
DI representations (i.e.,\emph{ just-right compressed DI representations}),
are balanced as in the case (b) of Figure \ref{fig:feature_illustration},
while for overly high $\lambda$, source compression from compressed
DI representations dominates and compromises the class separation
from general DI representations (i.e.,\emph{ overly-compressed DI
representations}).

Figure \ref{fig:val_vs_lambda} shows the source validation and target
accuracies when increasing $\lambda$ (i.e., encouraging the source
compression). It can be observed that both source validation accuracy
and target accuracy have the same pattern: increasing when setting
appropriate $\lambda$ for \emph{just-right compressed DI representations
}and compromising when setting overly high values $\lambda$ for \emph{overly-compressed
DI representations}. Figure \ref{fig:val_vs_step} shows in detail
the variation of the source validation accuracy for each specific
$\lambda$. In practice, we should encourage learning two kinds of
DI representations simultaneously by finding an appropriate trade-off
to balance them for working out just-right compressed DI representations\emph{.}

\subsection{Generalization Capacity on Various Target Domains\label{subsec:Generalization-Capacity}}
\begin{center}
\begin{figure}[H]
\begin{centering}
\subfloat[\label{fig:val_vs_lambda} Accuracy vs $\lambda$]{\begin{centering}
\includegraphics[width=0.29\textwidth]{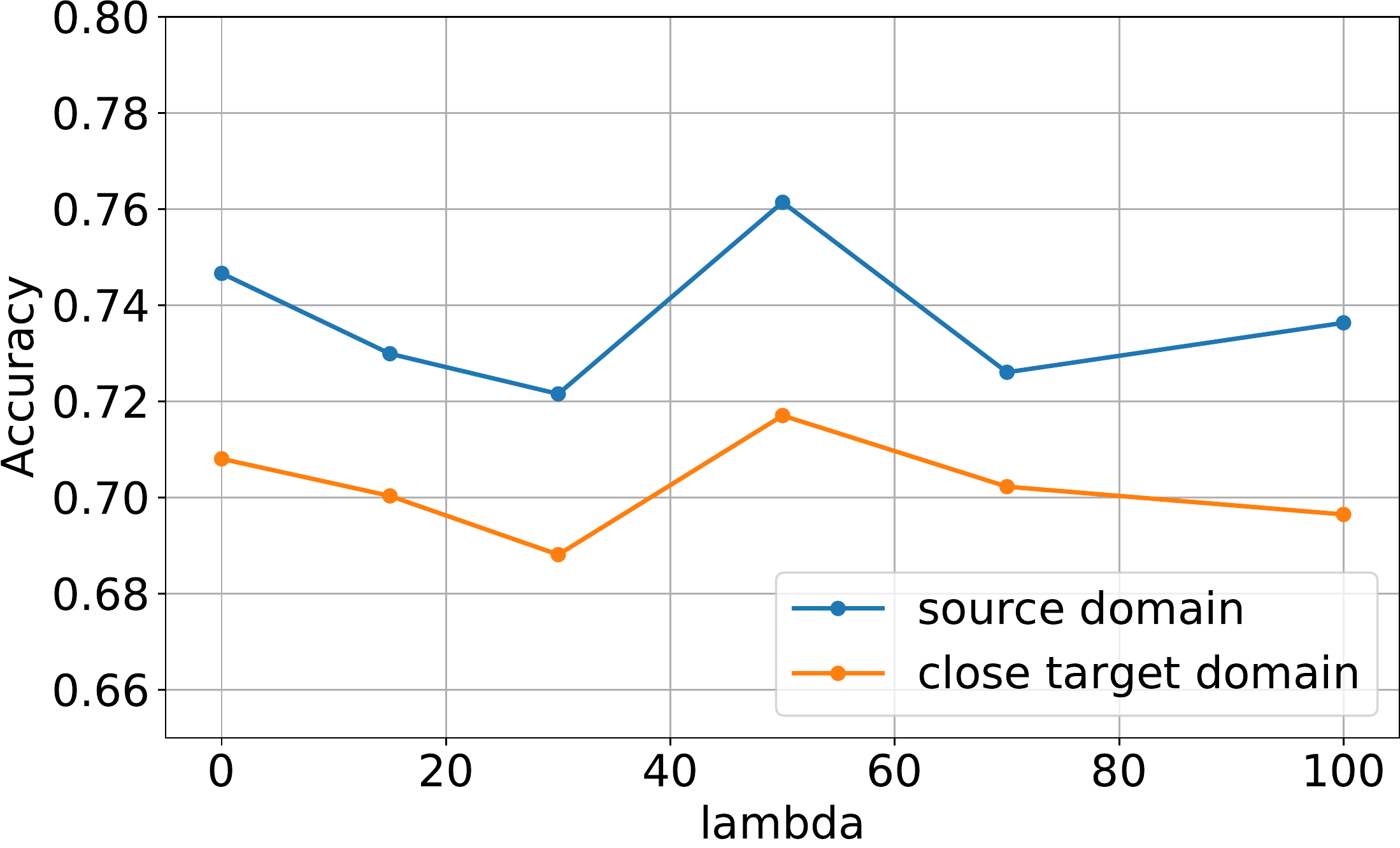} 
\par\end{centering}
}\hspace{50pt}\subfloat[\label{fig:val_vs_step} Accuracy vs iteration]{\begin{centering}
\includegraphics[width=0.31\textwidth]{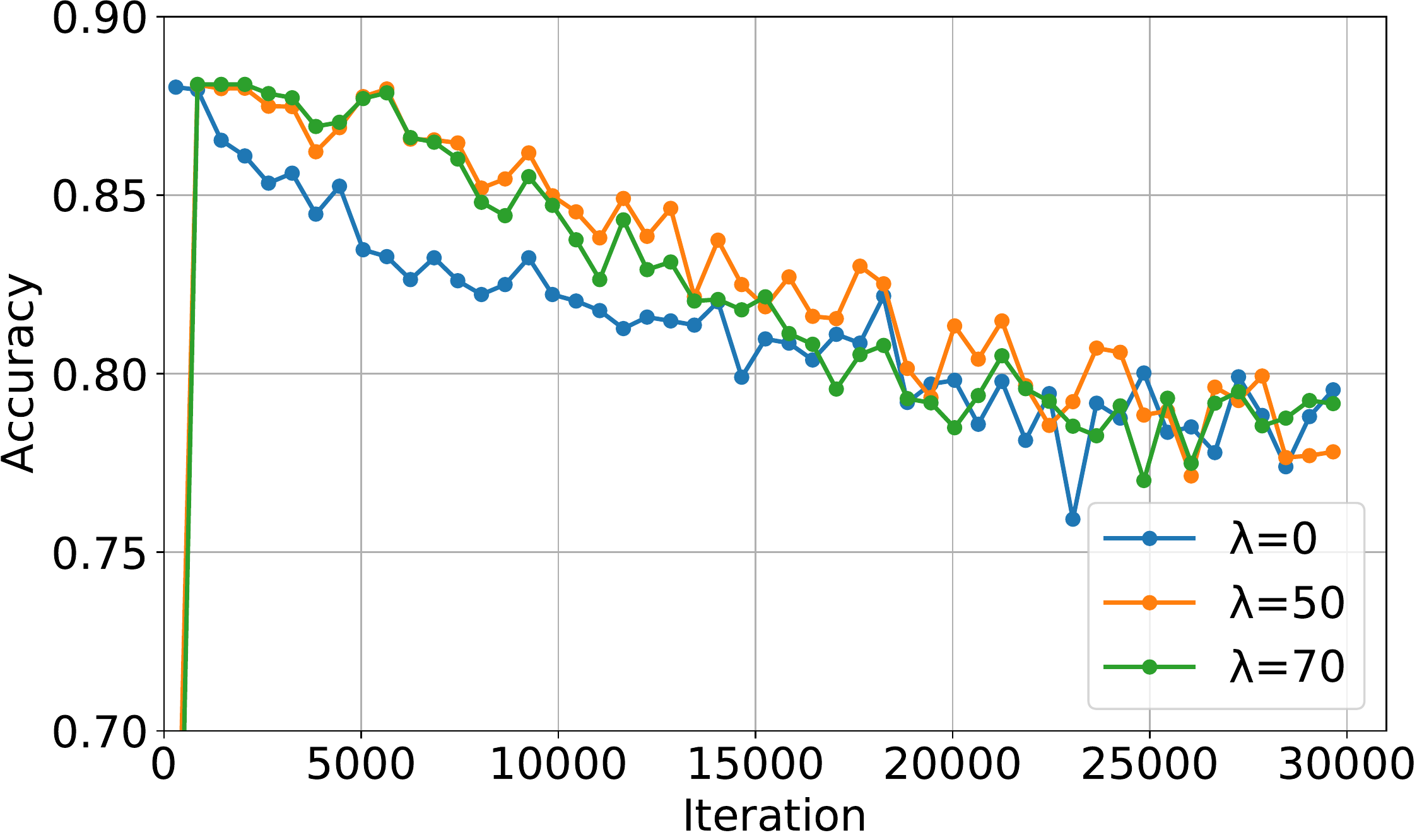} 
\par\end{centering}
}
\par\end{centering}
\caption{\label{fig:train_val_acc_vs_lambda_step} (\textbf{\ref{fig:val_vs_lambda}})
Source validation accuracy and target accuracy for close target domain
(see Section \ref{subsec:Generalization-Capacity}) over compression
strength. (\textbf{\ref{fig:val_vs_step})} Validation accuracy over
training step for different values of $\lambda$.}

\vspace{-30pt}
 
\end{figure}
\par\end{center}

\begin{center}
\begin{figure}[H]
\vspace{-20pt}

\begin{centering}
\subfloat[\label{fig:test_acc_vs_step_0}Far target domain with $\theta=0.05$.]{\centering{}\includegraphics[width=0.31\textwidth]{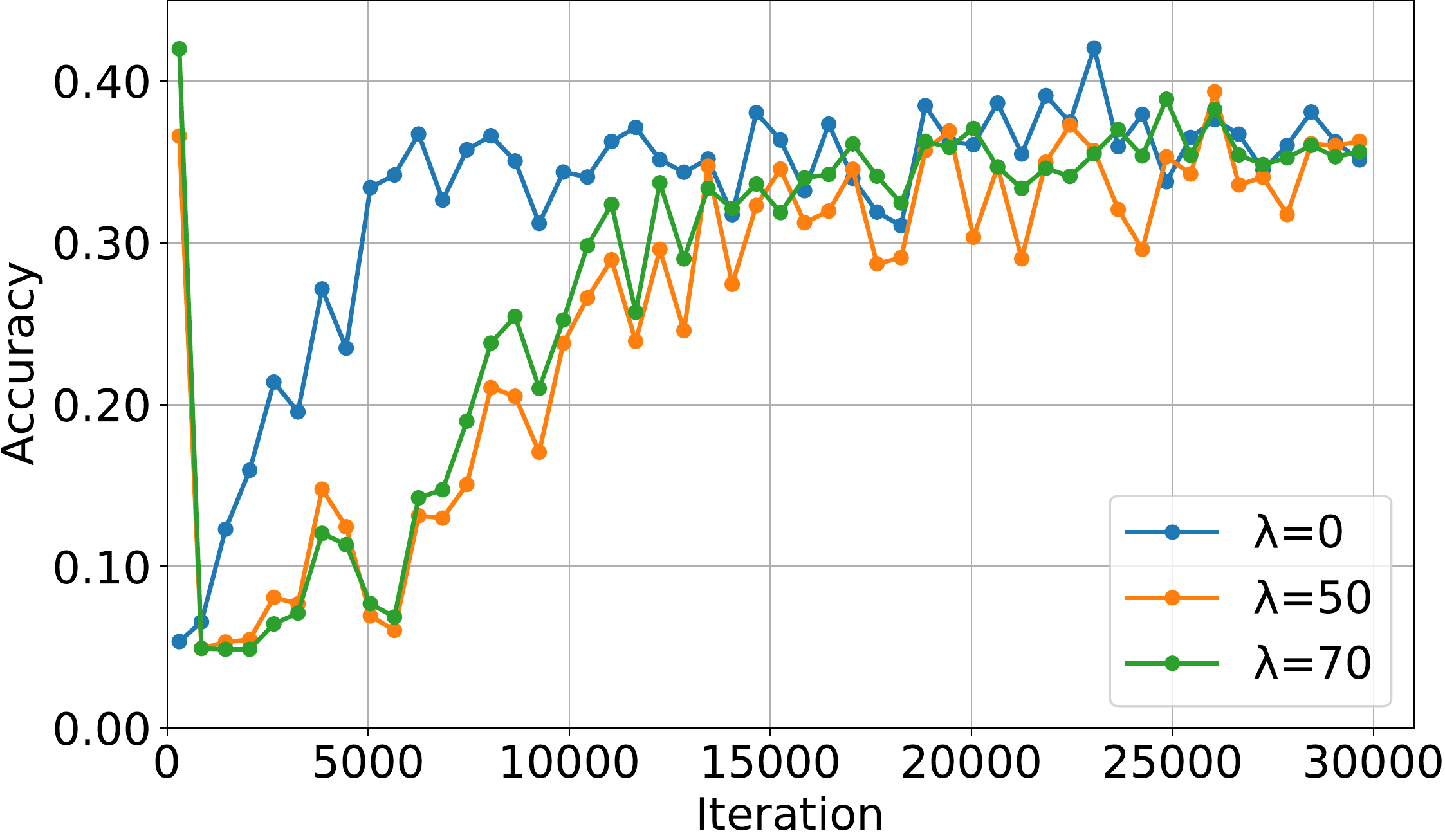}

}\hspace{50pt}\subfloat[\label{fig:test_acc_vs_step_2}Close target domain with $\theta=0.7$.]{\begin{centering}
\includegraphics[width=0.31\textwidth]{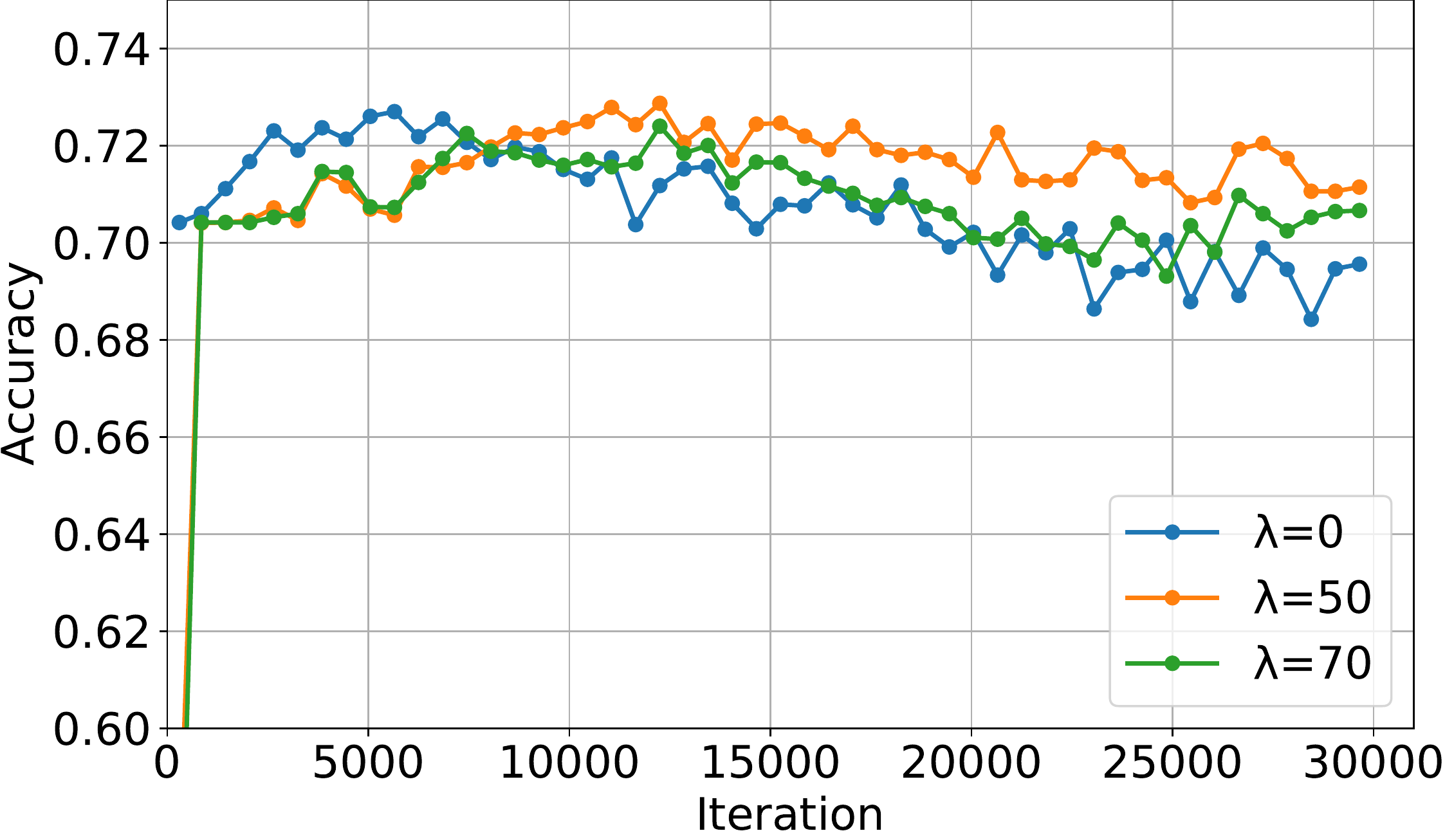} 
\par\end{centering}
}
\par\end{centering}
\caption{\label{fig:test_acc_vs_step} Accuracy of models trained with different
$\lambda$ on two test domains. For the far domain with $\theta=0.05$,
the model corresponding with general DI representations performs better
than that with more compact DI representations. For close domain with
$\theta=0.7$, the models with compression generalize significantly
better than those with general DI representations.}

\vspace{-15mm}
 
\end{figure}
\par\end{center}

In this experiment, we inspect the generalization capacity of the
models trained on the source domains with different compression strengths
$\lambda$ to different target domains (i.e., close and far ones).
As shown in Figure \ref{fig:test_acc_vs_step}, the target accuracy
is lower for the far target domain than for the close one regardless
of the source compression strengths $\lambda$. This observation aligns
with our upper-bounds developed in Theorems \ref{thm:general_dom_inv}
and \ref{thm:compact_dom_inv} in which larger representation discrepancy
increases the upper-bounds of the general target loss, hence more
likely hurting the target performance.

Another interesting observation is that for the far target domain,
the target accuracy for $\lambda=0$ peaks. This can be partly explained
as the general DI representations help to increase the span of source
latent representations for more chance to match a far target domain
(cf. Section \ref{subsec:General_features}). Specifically, the source-target
discrepancy $d_{1/2}\left(\mathbb{P}_{g}^{T},\mathbb{P}_{g}^{\pi}\right)$
could be small and the upper bound in Theorem \ref{thm:general_dom_inv}
is lower. In contrast, for the close target domain, the compressed
DI representations help to really improve the target performance,
while over-compression degrades the target performance, similar to
result on source domains. This is because the target domain is naturally
close to the source domains, the source and target latent representations
are already mixed up, but by encouraging the compressed DI representations,
we aim to learning more elegant representations for improving target
performance.

\section{Related Works}

\emph{Domain adaptation (DA)} has been intensively studied from both
theoretical \citep{DA_ben_david_2010_DAbound,DA_ben_david_impossibility_pmlr_2010,DA_Cortes_ALT2011_regression,DA_theo_zhao19a_on_learning_tradeoff,DA_mansour_2014_roburst_da,DA_redko_2017_theory_ot_for_da}
and empirical \citep{DA_Saito_NEURIPS2020,DA_pmlr-zhang19_Jordan_bridging,DA_ganin_DANN_2016,DA_Cortes_JMLR_2019_gen_disc,DA_Yan_2020_mixup_train}
perspectives. Notably, the pioneering work \citep{DA_ben_david_2010_DAbound}
and other theoretical works \citep{DA_redko_2017_theory_ot_for_da,DA_pmlr-zhang19_Jordan_bridging,DA_theo_zhao19a_on_learning_tradeoff,DA_johansson19a_support_invertibility,DA_mansour_colt09_DAbound_genloss,DA_Cortes_ALT2011_regression}
have investigated DA in various aspects which lay foundation for the
success of practical works \citep{DA_ganin_DANN_2016,DA_tzeng_cvpr2017_DANN_variants,DA_Saito_NEURIPS2020,DA_pmlr-zhang19_Jordan_bridging,DA_ganin_DANN_2016,DA_Cortes_JMLR_2019_gen_disc,DA_Yan_2020_mixup_train}.
\emph{Multiple-source domain adaptation (MSDA)} extends DA by gathering
training set from multiple source domains \citep{,MSDA_JMLR_crammer08a,MDA_mansour_MDA_limited_target_pmlr-21,MDA_Redko_mach_learn_2019_adaptability_analysis,MDA_redko_aistat19_ot_for_MDA,DA_redko_2017_theory_ot_for_da,MSDA_Duan_exploiting_web,MSDA_Zhao_NEURIPS2018_adversarial_msda,MSDA_Xu_2018_CVPR_deep_cooktail,MSDA_Peng_2019_ICCV_moment_matching}.
Different from DA with abundant theoretical works, theoretical study
in MSDA is significantly limited. Notably, the works in \citep{MDA_mansour_nips2009_ens_hypo,MDA_Mansour_renyi_div_UAI2009}
% The work in \cite{MDA_hoffman_nips2018_learn_ens} develops both theory and algorithm to learn such combination for different losses.
%\trungpq{
relies on assumptions about the existence of the same labeling function,
which can be used for all source domains; and target domain is an
unknown mixture of the source domains. They then show that there exists
a distributional weighted combination of these experts performing
well on target domain with loss at most $\epsilon$, given a set of
source expert hypotheses (with loss is at most on respective domain).
Hoffman et al. \citep{MDA_hoffman_nips2018_learn_ens} further develop
this idea to a more general case where different labeling functions
corresponding to different source domains, and the target domain is
arbitrary. Under this setting, there exists a distributional weighted
combination of source experts such that the loss on target domain
is bounded by a source loss term and a discrepancy term between target
domain and source mixture. On the other hand, Zhao et al. \citep{MSDA_Zhao_NEURIPS2018_adversarial_msda}
directly extend Ben-david et al.'s work with an improvement on the
sample complexity term. %}

\emph{Domain generalization (DG)} \citep{domainbed,DG_li2017_iccv_pacs,DG_Blanchard_NIPS2011,DG_huang2_eccv20_self_challenging}
is the most challenging setting among the three due to the unavailability
of target data. The studies in \citep{DG_Blanchard_NIPS2011,DG_deshmukh_2019}
use kernel method to find feature map which minimizes expected target
loss. Moreover, those in \citep{DG_Li_MMD_DANN_2018_CVPR,DG_Li_CDANN_2018_ECCV,DG_muandet_icml13_kernel,DG_ghifary_iccv2015,DG_Haohan_iclr_2019_Projecting_Superficial_Statistics,DG_dou_nips2019_model_agnostic_learning}
learn compressed DI representations by minimizing different types
of domain dissimilarity. Other notions of invariance are also proposed,
for example, \citep{OOD_IRM_arjovsky_2020} learns a latent space
on which representation distributions do not need to be aligned, but
share a common optimal hypothesis. Another work in \citep{DG_piratla_icml20a_pmlr_low_rank_decom}
uncovers domain-invariant hypothesis from low-rank decomposition of
domain-specific hypotheses. Moreover, there are other works which
try to strike a balance between two types of representation spaces
\citep{DG_chattopadhyay2020_learning_balance}.

\section{Conclusion}

In this paper, we derive theoretical bounds for target loss in MSDA
and DG settings which characterize two types of representation: general
DI representation for learning invariant classifier which works on
all source domains and compressed DI representation motivated from
reducing inter-domain representation discrepancy. We further characterize
the properties of these two representations, and develop a lower bound
on the target loss which governs the trade-off between learning them.
Finally, we conduct experiments on Colored MNIST dataset and real
dataset to illustrate our theoretical claims.

\clearpage{}

\bibliographystyle{acm}
\bibliography{2021-NeurIPS-On_DI_Rep_for_MSDA_DG_main}

\clearpage{}

\appendix This supplementary material provides proofs for theoretical results
stated in the main paper, as well as detailed experiment settings
and further experimental results. It is organized as follows
\begin{itemize}
\item Appendix A contains the proofs for the upper bounds of target loss
introduced in Section 2.2 of our main paper.
\item Appendix B contains the proofs for characteristics of two representations
as mentioned in Section 2.3 of our main paper.
\item Appendix C contains the proof for trade-off theorem discussed in Section
2.4 of our main paper.
\item Finally, in Appendix D, we present the generative detail of the Colored
MNIST dataset, the experimental setting together with additional results
for MSDA on Colored MNIST and DG on the real-world PACS dataset.
\end{itemize}

\section{Appendix A: Target Loss's Upper Bounds}

We begin with a crucial proposition for our theory development, which
further allows us to connect loss on feature space to loss on data
space.
\begin{prop}
\label{prop:latent_input} Let $\hat{f}:\mathcal{X}\mapsto\mathcal{Y}_{\Delta}$
where $\hat{f}=\hat{h}\circ g$ with $g:\mathcal{X}\mapsto\mathcal{Z}$
and $\hat{h}:\mathcal{Z}\mapsto\mathcal{Y}_{\Delta}$. Let $c:\mathcal{Z}\times\left[C\right]\mapsto\mathbb{R}$
be a positive function. 

i) For any $i\in\left[C\right]$, we have

\[
\int h\left(z,i\right)c\left(\hat{h}\left(z\right),i\right)p_{g}\left(z\right)dz=\int f\left(x,i\right)c\left(\hat{f}\left(x\right),i\right)p\left(x\right)dx,
\]

ii) For any $i\in\left[C\right]$, we have
\[
\int\left|h\left(z,i\right)-h'\left(z,i\right)\right|c\left(\hat{h}\left(z\right),i\right)p_{g}\left(z\right)dz\leq\int\left|f\left(x,i\right)-f'\left(x,i\right)\right|c\left(\hat{f}\left(x\right),i\right)p\left(x\right)dx,
\]

where $p$ is the density of a distribution $\mathbb{P}$ on $\mathcal{X}$,
$p_{g}$ is the density of the distribution $\mathbb{P}_{g}=g_{\#}\mathbb{P}$,
$f:\mathcal{X}\mapsto\mathcal{Y}_{\Delta}$ and $f':\mathcal{X}\mapsto\mathcal{Y}_{\Delta}$
are labeling functions, and $h:\mathcal{Z}\mapsto\mathcal{Y}_{\Delta}$
and $h':\mathcal{Z}\mapsto\mathcal{Y}_{\Delta}$ are labeling functions
on the latent space induced from $f$, i.e., $h(z)=\frac{\int_{g^{-1}\left(z\right)}f\left(x,i\right)p\left(x\right)dx}{\int_{g^{-1}\left(z\right)}p\left(x\right)dx}$
and $h'(z)=\frac{\int_{g^{-1}\left(z\right)}f'\left(x,i\right)p\left(x\right)dx}{\int_{g^{-1}\left(z\right)}p\left(x\right)dx}$.
\end{prop}

\begin{proof}
i) Using the definition of $h$, we manipulate the integral on feature
space as

\begin{align*}
\int_{\mathcal{Z}}h\left(z,i\right)c\left(\hat{h}\left(z\right),i\right)p_{g}\left(z\right)dz & \stackrel{\left(1\right)}{=}\int_{\mathcal{\mathcal{Z}}}c\left(\hat{h}\left(z\right),i\right)\frac{\int_{g^{-1}\left(z\right)}f\left(x,i\right)p\left(x\right)dx}{\int_{g^{-1}\left(z\right)}p\left(x\right)dx}\int_{g^{-1}\left(z\right)}p\left(x'\right)dx'dz\\
 & =\int_{\mathcal{\mathcal{Z}}}c\left(\hat{h}\left(z\right),i\right)\int_{g^{-1}\left(z\right)}f\left(x,i\right)p\left(x\right)dxdz\\
 & \stackrel{\left(2\right)}{=}\int_{\mathcal{\mathcal{Z}}}\int_{g^{-1}\left(z\right)}c\left(\hat{h}\left(g\left(x\right)\right),i\right)f\left(x,i\right)p\left(x\right)dxdz\\
 & \stackrel{\left(3\right)}{=}\int_{\mathcal{\mathcal{Z}}}\int_{\mathcal{X}}\mathbb{I}_{x\in g^{-1}\left(z\right)}c\left(\hat{h}\left(z\right),i\right)f\left(x,i\right)p\left(x\right)dxdz\\
 & \stackrel{\left(4\right)}{=}\int_{\mathcal{X}}\int_{\mathcal{\mathcal{Z}}}\mathbb{I}_{z=g\left(x\right)}c\left(\hat{h}\left(z\right),i\right)f\left(x,i\right)p\left(x\right)dxdz\\
 & =\int_{\mathcal{X}}c\left(\hat{h}\left(g\left(x\right)\right),i\right)f\left(x,i\right)p\left(x\right)dx\\
 & =\int_{\mathcal{X}}c\left(\hat{f}\left(x\right),i\right)f\left(x,i\right)p\left(x\right)dx.
\end{align*}

In $\left(1\right)$, we use the definition of push-forward distribution
$\int_{B}p_{g}\left(z\right)dz=\int_{B}dz\int_{g^{-1}\left(z\right)}p\left(x\right)dx$.
In $\left(2\right)$, $c\left(\hat{h}\left(z\right),i\right)$ can
be put inside the integral because $z=g\left(x\right)$ for any $x\in g^{-1}\left(z\right)$.
In $\left(3\right)$, the integral over restricted region is expanded
to all space with the help of $\mathbb{I}_{x\in g^{-1}\left(z\right)}$,
whose value is $1$ if $x\in g^{-1}\left(z\right)$ and $0$ otherwise.
In $\left(4\right)$, Fubini theorem is invoked to swap the integral
signs.

ii) Using the same technique, we have

\begin{align*}
 & \int\left|h\left(z,i\right)-h'\left(z,i\right)\right|c\left(\hat{h}\left(z\right),i\right)p_{g}\left(z\right)dz\\
 & =\int_{\mathcal{\mathcal{Z}}}c\left(\hat{h}\left(z\right),i\right)\frac{\left|\int_{g^{-1}\left(z\right)}f\left(x,i\right)p\left(x\right)dx-\int_{g^{-1}\left(z\right)}f'\left(x,i\right)p\left(x\right)dx\right|}{\int_{g^{-1}\left(z\right)}p\left(x'\right)dx'}\int_{g^{-1}\left(z\right)}p\left(x\right)dxdz\\
 & =\int_{\mathcal{\mathcal{Z}}}c\left(\hat{h}\left(z\right),i\right)\left|\int_{g^{-1}\left(z\right)}\left(f\left(x,i\right)-f'\left(x,i\right)\right)p\left(x\right)dx\right|dz\\
 & \leq\int_{\mathcal{\mathcal{Z}}}c\left(\hat{h}\left(z\right),i\right)\int_{g^{-1}\left(z\right)}\left|f\left(x,i\right)-f'\left(x,i\right)\right|p\left(x\right)dxdz\\
 & =\int_{\mathcal{\mathcal{Z}}}\int_{\mathcal{X}}\mathbb{I}_{x\in g^{-1}\left(z\right)}c\left(\hat{h}\left(z\right),i\right)\left|f\left(x,i\right)-f'\left(x,i\right)\right|p\left(x\right)dxdz\\
 & =\int_{\mathcal{X}}\int_{\mathcal{\mathcal{Z}}}\mathbb{I}_{z=g\left(x\right)}c\left(\hat{h}\left(g\left(x\right)\right),i\right)\left|f\left(x,i\right)-f'\left(x,i\right)\right|p\left(x\right)dxdz\\
 & =\int_{\mathcal{X}}\left|f\left(x,i\right)-f'\left(x,i\right)\right|c\left(\hat{f}\left(x\right),i\right)p\left(x\right)dx.
\end{align*}
\end{proof}
Proposition \ref{prop:latent_input} allows us to connect expected
loss on feature space and input space as in the following corollary.
\begin{cor}
\label{cor:feature-loss-input-loss}Consider a domain $\mathbb{D}=\left(\mathbb{P},f\right)$
with data distribution $\mathbb{P}$ and ground-truth labeling function
$f$. A hypothesis is $\hat{f}:\mathcal{X}\mapsto\mathcal{Y}_{\Delta}$,
where $\hat{f}=\hat{h}\circ g$ with $g:\mathcal{X}\mapsto\mathcal{Z}$
and $\hat{h}:\mathcal{Z}\mapsto\mathcal{Y}_{\Delta}$. Then, the labeling
function $f$ on input space induces a ground-truth labeling function
on feature space $h(z)=\frac{\int_{g^{-1}\left(z\right)}f\left(x,i\right)p\left(x\right)dx}{\int_{g^{-1}\left(z\right)}p\left(x\right)dx}$.
Let $\ell:\mathcal{Y}_{\Delta}\times\mathcal{Y}_{\Delta}\mapsto\mathbb{R}$
be the loss function, then expected loss can be calculated either
w.r.t. input space $\mathcal{L}\left(\hat{f},f,\mathbb{P}\right)=\int\ell\left(\hat{f}\left(x\right),f\left(x\right)\right)p\left(x\right)dx$
or w.r.t. feature space $\mathcal{L}\left(\hat{h},h,g_{\#}\mathbb{P}\right)=\int\ell\left(\hat{h}\left(z\right),h\left(z\right)\right)p_{g}\left(z\right)dz$.
If we assume the loss $\ell\left(\cdot,\cdot\right)$ has the formed
mentioned in the main paper, that is,

\[
\ell\left(u,v\right)=\sum_{i=1}^{C}l\left(u,i\right)v_{i},
\]

for any two simplexes $u,v\in\mathcal{Y}_{\Delta}$, where $u=\left[u_{i}\right]_{i=1}^{C}$
and $v=\left[v_{i}\right]_{i=1}^{C}$. Then the losses w.r.t. input
space and feature space are the same, i.e.,

\[
\mathcal{L}\left(\hat{h},h,g_{\#}\mathbb{P}\right)=\mathcal{L}\left(\hat{f},f,\mathbb{P}\right).
\]
\end{cor}

\begin{proof}
We derive as

\textit{
\[
\begin{aligned}\mathcal{L}\left(\hat{h},h,g_{\#}\mathbb{P}\right) & =\int_{\mathcal{Z}}\ell\left(\hat{h}\left(z\right),h\left(z\right)\right)p\left(z\right)dz=\sum_{i=1}^{C}\int_{\mathcal{Z}}l\left(\hat{h}\left(z\right),i\right)h\left(z,i\right)p\left(z\right)dz\\
 & \stackrel{(1)}{=}\sum_{i=1}^{C}\int_{\mathcal{X}}l\left(\hat{f}\left(x\right),i\right)f\left(x,i\right)p\left(x\right)dx=\int_{\mathcal{X}}\ell\left(\hat{f}\left(x\right),f\left(x\right)\right)p\left(x\right)dx\\
 & =\mathcal{L}\left(\hat{f},f,\mathbb{P}\right).
\end{aligned}
\]
}

Here we have $\stackrel{(1)}{=}$ by using (i) in Proposition \ref{prop:latent_input}
with $c\left(\beta,i\right)=l\left(\beta,i\right)$ for any $\beta\in\mathcal{Y}_{\Delta}$
and $i\in\left[C\right]$.
\end{proof}
Normally, target loss $\mathcal{L}\left(\hat{f},f^{T},\mathbb{P}^{T}\right)$
is bounded by source loss $\mathcal{L}\left(\hat{f},f^{S},\mathbb{P}^{S}\right)$,
a label shift term $LS\left(f^{T},f^{S}\right)$, and a data shift
term $DS\left(\mathbb{P}^{T},\mathbb{P}^{S}\right)$ \citep{DA_ben_david_2010_DAbound}.
Here, this kind of bound is developed using data distribution $\mathbb{P}$
on input space and labeling function $f$ from input to label space,
which are not convenient in understanding representation learning,
since $\mathbb{P}^{T},\mathbb{P}^{S}$ are data nature and therefore
fixed. In order to relate target loss to properties of learned representations,
another bound in which $\mathcal{L}\left(\hat{h},h,\mathbb{P}_{g}\right)$
is the loss w.r.t. feature space and $DS\left(\mathbb{P}_{g}^{T},\mathbb{P}_{g}^{S}\right)$
is the data shift on feature space is more favorable. However, this
naive approach presents a pitfall, since the loss $\mathcal{L}\left(\hat{h},h,\mathbb{P}_{g}\right)$
is not identical to the loss w.r.t. input space $\mathcal{L}\left(\hat{f},f,\mathbb{P}\right)$,
which is of ultimate interest, e.g., to-be-bounded target loss $\mathcal{L}\left(\hat{f},f^{T},\mathbb{P}^{T}\right)$,
or to-be-minimized source loss $\mathcal{L}\left(\hat{f},f^{S},\mathbb{P}^{S}\right)$.
Using the previous proposition and corollary, we could bridge this
gap and develop a target bound connecting both data space and feature
space.
\begin{thm}
\label{apx_thm:general_dom_inv-1}(\textbf{Theorem 1 in the main paper})
Consider a mixture of source domains $\mathbb{D}^{\pi}=\sum_{i=1}^{K}\pi_{i}\mathbb{D}^{S,i}$
and the target domain $\mathbb{D}^{T}$. Let $\ell$ be any loss function
upper-bounded by a positive constant $L$. For any hypothesis $\hat{f}:\mathcal{X}\mapsto\mathcal{Y}_{\Delta}$
where $\hat{f}=\hat{h}\circ g$ with $g:\mathcal{X}\mapsto\mathcal{Z}$
and $\hat{h}:\mathcal{Z}\mapsto\mathcal{Y}_{\Delta}$, the target
loss on input space is upper bounded 
\begin{equation}
\begin{aligned}\mathcal{L}\left(\hat{f},\mathbb{D}^{T}\right)\leq\sum_{i=1}^{K}\pi_{i}\mathcal{L}\left(\hat{f},\mathbb{D}^{S,i}\right)+L\max_{i\in[K]}\mathbb{E}_{\mathbb{P}^{S,i}}\left[\|\Delta p^{i}(y|x)\|_{1}\right]+L\sqrt{2}\,d_{1/2}\left(\mathbb{P}_{g}^{T},\mathbb{P}_{g}^{\pi}\right)\end{aligned}
,\label{apx_eq:input_bound_1}
\end{equation}
where $\Delta p^{i}(y|x):=\left[\left|f^{T}(x,y)-f^{S,i}(x,y)\right|\right]_{y=1}^{C}$
is the absolute of single point label shift on input space between
source domain $\mathbb{D}^{S,i}$, the target domain $\mathbb{D}^{T}$,
$[K]:=\left\{ 1,2,...,K\right\} $, and the \textbf{feature distribution
of the source mixture} $\mathbb{P}_{g}^{\pi}:=\sum_{i=1}^{K}\pi_{i}\mathbb{P}_{g}^{S,i}$.
\end{thm}

\begin{proof}
First, consider the hybrid domain $\mathbb{D}_{g}^{h,i}=\left(\mathbb{P}_{g}^{\pi},h^{T}\right)$,
with $\mathbb{P}_{g}^{\pi}:=\sum_{i=1}^{K}\pi_{i}\mathbb{P}_{g}^{S,i}$
be the feature distribution of the source mixture, and $h^{T}$ is
the induced ground-truth labeling function of target domain. The loss
on the hybrid domain is then upper bounded by the loss on source mixture
and a label shift term. We derive as follows:
\[
\begin{aligned}\mathcal{L}\left(\hat{h},\mathbb{D}_{g}^{hi}\right) & =\int\ell\left(\hat{h}\left(z\right),h^{T}\left(z\right)\right)p_{g}^{\pi}\left(z\right)dz\\
 & =\int\ell\left(\hat{h}\left(z\right),h^{T}\left(z\right)\right)\sum_{i=1}^{K}\pi_{i}p_{g}^{S,i}\left(z\right)dz=\sum_{i=1}^{K}\pi_{i}\int\ell\left(\hat{h}\left(z\right),h^{T}\left(z\right)\right)p_{g}^{S,i}\left(z\right)dz\\
 & \leq\sum_{i=1}^{K}\pi_{i}\int\ell\left(\hat{h}\left(z\right),h^{S,i}\left(z\right)\right)p_{g}^{S,i}\left(z\right)dz\\
 & +\sum_{i=1}^{K}\pi_{i}\int\left|\ell\left(\hat{h}\left(z\right),h^{T}\left(z\right)\right)-\ell\left(\hat{h}\left(z\right),h^{S,i}\left(z\right)\right)\right|p_{g}^{S,i}\left(z\right)dz.
\end{aligned}
\]

Firstly, using Corollary \ref{cor:feature-loss-input-loss}, the loss
terms on feature space and input space are equal

\[
\sum_{i=1}^{K}\pi_{i}\int\ell\left(\hat{h}\left(z\right),h^{S,i}\left(z\right)\right)p_{g}^{S,i}\left(z\right)dz=\sum_{i=1}^{K}\pi_{i}\mathcal{L}\left(\hat{f},f^{S,i},\mathbb{P}^{S,i}\right)
\]

Secondly, the difference term, can be transformed into label shift
on input space using Proposition \ref{prop:latent_input}

\[
\begin{aligned} & \sum_{i=1}^{K}\pi_{i}\int_{\mathcal{Z}}\left|\ell\left(\hat{h}\left(z\right),h^{T}\left(z\right)\right)-\ell\left(\hat{h}\left(z\right),h^{S,i}\left(z\right)\right)\right|p_{g}^{S,i}\left(z\right)dz\\
 & \leq\sum_{i=1}^{K}\pi_{i}\sum_{j=1}^{C}\int_{\mathcal{Z}}\ell\left(\hat{h}\left(z\right),j\right)\left|h^{T}\left(z,j\right)-h^{S,i}\left(z,j\right)\right|p_{g}^{S,i}\left(z\right)dz\\
 & =L\sum_{i=1}^{K}\pi_{i}\sum_{j=1}^{C}\int_{\mathcal{Z}}\left|h^{T}\left(z,j\right)-h^{S,i}\left(z,j\right)\right|p_{g}^{S,i}\left(z\right)dz\\
 & \stackrel{(1)}{\leq}L\sum_{i=1}^{K}\pi_{i}\sum_{j=1}^{C}\int_{\mathcal{X}}\left|f^{T}\left(z,j\right)-f^{S,i}\left(z,j\right)\right|p^{S,i}\left(x\right)dx\\
 & \leq L\max_{i\in[K]}\mathbb{E}_{\mathbb{P}^{S,i}}\left[\|\Delta p^{i}\left(y|x\right)\|_{1}\right].
\end{aligned}
\]

Here we note that $\stackrel{(1)}{\leq}$ results from (ii) in Proposition
\ref{prop:latent_input} with $c\left(\beta,i\right)=1$ for any $\beta\in\mathcal{Y}_{\Delta}$
and $i\in\left[C\right]$. Furthermore, $\Delta p^{i}(y|x):=\left[\left|f^{T}(x,y)-f^{S,i}(x,y)\right|\right]_{y=1}^{C}$
is the absolute single point label shift on input space between the
source domain $\mathbb{D}^{S,i}$ and the target domain $\mathbb{D}^{T}$.

With these two terms, we have the upper bound for hybrid domain as

\[
\mathcal{L}\left(\hat{h},\mathbb{D}_{g}^{hy}\right)\leq\sum_{i=1}^{K}\pi_{i}\mathcal{L}\left(\hat{f},\mathbb{D}^{S,i}\right)+L\max_{i\in[K]}\mathbb{E}_{\mathbb{P}^{S,i}}\left[\|\Delta p^{i}(y|x)\|_{1}\right].
\]

Next, we relate the loss on target $\mathbb{D}_{g}^{T}$ to hybrid
domain $\mathbb{D}_{g}^{hy}$, which differs only at the feature marginals.

\begin{gather*}
\begin{aligned}\left|\mathcal{L}\left(\hat{h},\mathbb{D}_{g}^{T}\right)-\mathcal{L}\left(\hat{h},\mathbb{D}_{g}^{hy}\right)\right| & =\left|\int\ell\left(\hat{h}(z),h^{T}(z)\right)\left(p_{g}^{T}(z)-p_{g}^{\pi}(z)\right)dz\right|\\
 & \leq\int L\left|p_{g}^{T}(z)-p_{g}^{\pi}(z)\right|dz\\
 & \leq L\int\left|\sqrt{p_{g}^{T}(z)}+\sqrt{p_{g}^{\pi}(z)}\right|\left|\sqrt{p_{g}^{T}(z)}-\sqrt{p_{g}^{\pi}(z)}\right|dz\\
 & \leq L\left[\int\left(\sqrt{p_{g}^{T}(z)}+\sqrt{p_{g}^{\pi}(z)}\right)^{2}dz\right]^{1/2}\left[\int\left(\sqrt{p_{g}^{T}(z)}-\sqrt{p_{g}^{\pi}(z)}\right)^{2}dz\right]^{1/2}\\
 & \leq\frac{L}{\sqrt{2}}\left[\int\left(p_{g}^{T}(z)+p_{g}^{\pi}(z)+2\sqrt{p_{g}^{T}(z)p_{g}^{\pi}(z)}\right)dz\right]^{1/2}\\
 & \times\left[2\int\left(\sqrt{p_{g}^{T}(z)}-\sqrt{p_{g}^{\pi}(z)}\right)^{2}dz\right]^{1/2}\\
 & \leq\frac{L}{\sqrt{2}}\left[2+2\left(\int p_{g}^{T}(z)dz\int p_{g}^{\pi}(z)dz\right)^{1/2}\right]^{1/2}d_{1/2}\left(\mathbb{P}_{g}^{T},\mathbb{P}_{g}^{\pi}\right)\\
 & \leq L\sqrt{2}d_{1/2}\left(\mathbb{P}_{g}^{T},\mathbb{P}_{g}^{\pi}\right).
\end{aligned}
\end{gather*}
In the above proof, we repeatedly invoke Cauchy-Schwartz inequality
$\left|\int f(z)g(z)dz\right|^{2}\leq\int\left|f(z)\right|^{2}dz\int\left|g(z)\right|^{2}dz$.
Moreover, for the sake of completeness, we reintroduce the definition
of the square root Hellinger distance 
\[
d_{1/2}\left(\mathbb{P}_{g}^{T},\mathbb{P}_{g}^{\pi}\right)=\left[2\int\left(\sqrt{p_{g}^{T}(z)}-\sqrt{p_{g}^{\pi}(z)}\right)^{2}dz\right]^{1/2}.
\]

To this end, we obtain the upper bound for target loss related to
loss on souce mixture, a label shift term on input space, and a data
shift term between target domain and source mixture on feature space.

\begin{equation}
\mathcal{L}\left(\hat{h},\mathbb{D}_{g}^{T}\right)\leq\sum_{i=1}^{K}\pi_{i}\mathcal{L}\left(\hat{f},\mathbb{D}^{S,i}\right)+L\max_{i\in[K]}\mathbb{E}_{\mathbb{P}^{S,i}}\left[\|\Delta p^{i}(y|x)\|_{1}\right]+L\sqrt{2}d_{1/2}\left(\mathbb{P}_{g}^{T},\mathbb{P}_{g}^{\pi}\right).\label{eq:latent_bound}
\end{equation}

Finally, using the fact that loss on feature space equal loss on input
space (Corollary \ref{cor:feature-loss-input-loss}), we have

\[
\mathcal{L}\left(\hat{f},\mathbb{D}^{T}\right)\leq\sum_{i=1}^{K}\pi_{i}\mathcal{L}\left(\hat{f},\mathbb{D}^{S,i}\right)+L\max_{i\in[K]}\mathbb{E}_{\mathbb{P}^{S,i}}\left[\|\Delta p^{i}(y|x)\|_{1}\right]+L\sqrt{2}d_{1/2}\left(\mathbb{P}_{g}^{T},\mathbb{P}_{g}^{\pi}\right).
\]

That concludes our proof.
\end{proof}
This bound is novel since it relates loss on input space and data
shift on feature space. This allows us to further investigate how
source-source compression and source-target compression affect learning.
First, we prove a lemma showing decomposition of data shift between
target domain and source mixture $d_{1/2}\left(\mathbb{P}_{g}^{T},\mathbb{P}_{g}^{\pi}\right)$
to a sum of data shifts between target domain and source domains.
\begin{lem}
\label{lem:decompose-mixture-distance}Given a source mixture and
a target domain, we have the following

\[
d_{1/2}\left(\mathbb{P}_{g}^{T},\mathbb{P}_{g}^{\pi}\right)\leq\sum_{j=1}^{K}\sqrt{\pi_{j}}d_{1/2}\left(\mathbb{P}_{g}^{T},\mathbb{P}_{g}^{S,j}\right)
\]
\end{lem}

\begin{proof}
Firstly, we observe that

\[
\begin{aligned}d_{1/2}\left(\mathbb{P}_{g}^{T},\mathbb{P}_{g}^{\pi}\right) & =\left[2\int\left(\sqrt{p_{g}^{T}(z)}-\sqrt{p_{g}^{\pi}(z)}\right)^{2}dz\right]^{1/2}\\
 & =\left[2\int\left(p_{g}^{T}(z)+p_{g}^{\pi}(z)-2\sqrt{p_{g}^{T}(z)p_{g}^{\pi}(z)}\right)dz\right]^{1/2}.
\end{aligned}
\]

Secondly, we use Cauchy-Schwartz inequality to obtain 

\[
\begin{aligned}p_{g}^{T}(z)p_{g}^{\pi}(z) & =\left(\sum_{j=1}^{K}\pi_{j}p_{g}^{T}(z)\right)\left(\sum_{j=1}^{K}\pi_{j}p_{g}^{S,j}(z)\right)\\
 & \geq\left(\sum_{j=1}^{K}\pi_{j}\sqrt{p_{g}^{T}(z)p_{g}^{S,j}(z)}\right)^{2}.
\end{aligned}
\]

Therefore,we arrive at

\[
\begin{aligned}d_{1/2}\left(\mathbb{P}_{g}^{T},\mathbb{P}_{g}^{\pi}\right) & \leq\left[2\int\left(\sum_{j=1}^{K}\pi_{j}p_{g}^{T}(z)+\sum_{j=1}^{K}\pi_{j}p_{g}^{S,j}(z)-2\sum_{j=1}^{K}\pi_{j}\sqrt{p_{g}^{T}(z)p_{g}^{S,j}(z)}\right)dz\right]^{1/2}\\
 & =\left[\sum_{j=1}^{K}\pi_{j}2\int\left(p_{g}^{T}(z)+p_{g}^{S,j}(z)-2\sqrt{p_{g}^{T}(z)p_{g}^{S,j}(z)}\right)dz\right]^{1/2}\\
 & \leq\sum_{j=1}^{K}\left[\pi_{j}2\int\left(p_{g}^{T}(z)+p_{g}^{S,j}(z)-2\sqrt{p_{g}^{T}(z)p_{g}^{S,j}(z)}\right)dz\right]^{1/2}\\
 & =\sum_{j=1}^{K}\sqrt{\pi_{j}}d_{1/2}\left(\mathbb{P}_{g}^{T},\mathbb{P}_{g}^{S,j}\right).
\end{aligned}
\]
\end{proof}
Now we are ready to prove the bound which motivate compressed DI representation.
\begin{thm}
\label{apx_thm:compact_dom_inv}(\textbf{Theorem 3 in the main paper})
Consider mixture of source domains $\mathbb{D}^{\pi}=\sum_{i=1}^{K}\pi_{i}\mathbb{D}^{S,i}$
and target domain $\mathbb{D}^{T}$. Let $\ell$ be any loss function
upper-bounded by a positive constant $L$. For any hypothesis $\hat{f}:\mathcal{X}\mapsto\mathcal{Y}_{\Delta}$
where $\hat{f}=\hat{h}\circ g$ with $g:\mathcal{X}\mapsto\mathcal{Z}$
and $\hat{h}:\mathcal{Z}\mapsto\mathcal{Y}_{\Delta}$, the target
loss on input space is upper bounded 
\end{thm}

\textit{
\begin{equation}
\begin{aligned}\mathcal{L}\left(\hat{f},\mathbb{D}^{T}\right) & \leq\sum_{i=1}^{K}\pi_{i}\mathcal{L}\left(\hat{f},\mathbb{D}^{S,i}\right)+L\max_{i\in[K]}\mathbb{E}_{\mathbb{P}^{S,i}}\left[\|\Delta p^{i}(y|x)\|_{1}\right]\\
 & +\sum_{i=1}^{K}\sum_{j=1}^{K}\frac{L\sqrt{2\pi_{j}}}{K}\left(d_{1/2}\left(\mathbb{P}_{g}^{T},\mathbb{P}_{g}^{S,i}\right)+d_{1/2}\left(\mathbb{P}_{g}^{S,i},\mathbb{P}_{g}^{S,j}\right)\right)
\end{aligned}
\label{apx_eq:compact_feature_inq}
\end{equation}
}
\begin{proof}
In the previous Theorem \ref{apx_thm:general_dom_inv-1}, the upper
bound for target loss is

\[
\mathcal{L}\left(\hat{f},\mathbb{D}^{T}\right)\leq\sum_{i=1}^{K}\pi_{i}\mathcal{L}\left(\hat{f},\mathbb{D}^{S,i}\right)+L\max_{i\in[K]}\mathbb{E}_{\mathbb{P}^{S,i}}\left[\|\Delta p^{i}(y|x)\|_{1}\right]+L\sqrt{2}d_{1/2}\left(\mathbb{P}_{g}^{T},\mathbb{P}_{g}^{\pi}\right).
\]

Using Lemma \ref{lem:decompose-mixture-distance}, we have

\[
\begin{aligned}d_{1/2}\left(\mathbb{P}_{g}^{T},\mathbb{P}_{g}^{\pi}\right) & \leq\sum_{j=1}^{K}\sqrt{\pi_{j}}d_{1/2}\left(\mathbb{P}_{g}^{T},\mathbb{P}_{g}^{S,j}\right)\end{aligned}
\]

Next, we use the triangle inequality for square root Hellinger distance

\[
\begin{aligned}d_{1/2}\left(\mathbb{P}_{g}^{T},\mathbb{P}_{g}^{\pi}\right) & \leq\sum_{j=1}^{K}\sqrt{\pi_{j}}d_{1/2}\left(\mathbb{P}_{g}^{T},\mathbb{P}_{g}^{S,j}\right)\\
 & \leq\sum_{j=1}^{K}\sqrt{\pi_{j}}\left(d_{1/2}\left(\mathbb{P}_{g}^{T},\mathbb{P}_{g}^{S,i}\right)+d_{1/2}\left(\mathbb{P}_{g}^{S,i},\mathbb{P}_{g}^{S,j}\right)\right)
\end{aligned}
\]

Therefore, by average over all $\mathbb{P}_{g}^{T},\mathbb{P}_{g}^{S,i}$
pairs,

\[
d_{1/2}\left(\mathbb{P}_{g}^{T},\mathbb{P}_{g}^{\pi}\right)=\sum_{i=1}^{K}\frac{1}{K}d_{1/2}\left(\mathbb{P}_{g}^{T},\mathbb{P}_{g}^{\pi}\right)\leq\sum_{i=1}^{K}\sum_{j=1}^{K}\frac{\sqrt{\pi_{j}}}{K}\left(d_{1/2}\left(\mathbb{P}_{g}^{T},\mathbb{P}_{g}^{S,i}\right)+d_{1/2}\left(\mathbb{P}_{g}^{S,i},\mathbb{P}_{g}^{S,j}\right)\right)
\]

We obtain the conclusion of our proof

\[
\begin{aligned}\mathcal{L}\left(\hat{f},\mathbb{D}^{T}\right) & \leq\sum_{i=1}^{K}\pi_{i}\mathcal{L}\left(\hat{f},\mathbb{D}^{S,i}\right)+L\max_{i\in[K]}\mathbb{E}_{\mathbb{P}^{S,i}}\left[\|\Delta p^{i}(y|x)\|_{1}\right]\\
 & +\sum_{i=1}^{K}\sum_{j=1}^{K}\frac{L\sqrt{2\pi_{j}}}{K}\left(d_{1/2}\left(\mathbb{P}_{g}^{T},\mathbb{P}_{g}^{S,i}\right)+d_{1/2}\left(\mathbb{P}_{g}^{S,i},\mathbb{P}_{g}^{S,j}\right)\right)
\end{aligned}
\]
\end{proof}

\section{Appendix B: DI Representation's Characteristics}

\subsection{General Domain-Invariant Representations}

In the main paper, we defined general DI representation via minimization
of source loss $\min_{g\in\mathcal{G}}\min_{\hat{h}\in\mathcal{H}}\sum_{i=1}^{K}\pi_{i}\mathcal{L}\left(\hat{h},h^{S,i},\mathbb{P}_{g}^{S,i}\right)$.
We then proposed to view the optimization problem $\min_{\hat{h}\in\mathcal{H}}\sum_{i=1}^{K}\pi_{i}\mathcal{L}\left(\hat{h},h^{S,i},\mathbb{P}_{g}^{S,i}\right)$
as calculating a type of divergence, i.e., hypothesis-aware divergence.
To understand the connection between the two, we first consider the
classification problem where samples are drawn from a mixture $z\sim\mathbb{Q}^{\alpha}=\sum_{i=1}^{C}\alpha_{i}\mathbb{Q}_{i}$,
with $\mathbb{Q}_{i}$ defined on $\mathcal{Z}$ and density being
$q_{i}\left(z\right)$, and the task is to predict which distributions
$\mathbb{Q}_{1},...,\mathbb{Q}_{C}$ the samples originate from, i.e.,
labels being $1,\ldots,C$. Here, the hypothesis class $\mathcal{H}$
is assumed to have infinite capacity, and the objective is to minimize
$\min_{\hat{h}\in\mathcal{H}}\mathcal{L}_{\mathbb{Q}_{1:C}}^{\alpha}\left(\hat{h}\right)=\min_{\hat{h}\in\mathcal{H}}\sum_{i=1}^{C}\alpha_{i}\mathcal{L}\left(\hat{h},\mathbb{Q}_{i}\right)$.

\subsubsection{Hypothesis-Aware Divergence}
\begin{thm}
\label{apx_thm:op_distance} (\textbf{Theorem 5 in the main paper})
Assuming the hypothesis class $\mathcal{H}$ has infinite capacity,
we define the hypothesis-aware divergence for multiple distributions
as

\begin{equation}
D^{\alpha}\left(\mathbb{Q}_{1},...,\mathbb{Q}_{C}\right)=-\min_{\hat{h}\in\mathcal{H}}\mathcal{L}_{\mathbb{Q}_{1:C}}^{\alpha}\left(\hat{h}\right)+\inf_{\beta\in\mathcal{Y}_{\Delta}}\left(\sum_{i=1}^{C}l\left(\beta,i\right)\alpha_{i}\right).\label{apx_eq:loss_distance_connection}
\end{equation}

This divergence is a proper divergence among $\mathbb{Q}_{1},...,\mathbb{Q}_{C}$
in the sense that $D^{\alpha}\left(\mathbb{Q}_{1},...,\mathbb{Q}_{C}\right)\geq0$
for all $\mathbb{Q}_{1},...,\mathbb{Q}_{C}$ and $\alpha\in\mathcal{Y}_{\simplex}$,
and $D^{\alpha}\left(\mathbb{Q}_{1},...,\mathbb{Q}_{C}\right)=0$
if $\mathbb{Q}_{1}=...=\mathbb{Q}_{C}$.
\end{thm}

\begin{proof}
Data is sampled from the mixture $\mathbb{Q}^{\alpha}$ by firstly
sampling domain index $i\sim Cat\left(\alpha\right)$, then sampling
data $z\sim\mathbb{Q}_{i}$ and label with $i$. We examine the the
total expected loss for any hypothesis $\hat{h}\in\mathcal{H}$, which
is

\[
\begin{aligned}\mathcal{L}_{\mathbb{Q}_{1:C}}^{\alpha}\left(\hat{h}\right) & :=\sum_{i=1}^{C}\alpha_{i}\mathcal{L}\left(\hat{h},\mathbb{Q}_{i}\right)\\
 & =\sum_{i=1}^{C}\alpha_{i}\int l\left(\hat{h}\left(z\right),i\right)q_{i}\left(z\right)dz
\end{aligned}
\]

We would like to minimize this loss, which leads to

\begin{equation}
\begin{aligned}\min_{\hat{h}\in\mathcal{H}}\mathcal{L}_{\mathbb{Q}_{1:C}}^{\alpha}\left(\hat{h}\right) & =\min_{\hat{h}\in\mathcal{H}}\sum_{i=1}^{C}\alpha_{i}\int l\left(\hat{h}\left(z\right),i\right)q_{i}\left(z\right)dz\\
 & \stackrel{\left(1\right)}{=}\min_{\hat{h}\in\mathcal{H}}\int\left(\sum_{i=1}^{C}\alpha_{i}l\left(\hat{h}\left(z\right),i\right)\frac{q_{i}\left(z\right)}{q^{\alpha}\left(z\right)}\right)q^{\alpha}\left(z\right)dz\\
 & \stackrel{\left(2\right)}{=}\int\min_{\hat{h}\in\mathcal{H}}\left(\sum_{i=1}^{C}\alpha_{i}l\left(\hat{h}\left(z\right),i\right)\frac{q_{i}\left(z\right)}{q^{\alpha}\left(z\right)}\right)q^{\alpha}\left(z\right)dz\\
 & =\int\min_{\beta\in\mathcal{Y}_{\Delta}}\left(\sum_{i=1}^{C}\alpha_{i}l\left(\beta,i\right)\frac{q_{i}\left(z\right)}{q^{\alpha}\left(z\right)}\right)q^{\alpha}\left(z\right)dz\\
 & \stackrel{\left(3\right)}{\leq}\min_{\beta\in\mathcal{Y}_{\Delta}}\left(\int\sum_{i=1}^{C}\alpha_{i}l\left(\beta,i\right)q_{i}\left(z\right)dz\right)\\
 & =\min_{\beta\in\mathcal{Y}_{\Delta}}\left(\sum_{i=1}^{C}\alpha_{i}l\left(\beta,i\right)\right).
\end{aligned}
\label{eq:convex_inequality}
\end{equation}

For $\left(1\right)$, $q^{\alpha}\left(z\right)=\sum_{i=1}^{C}\alpha_{i}q_{i}\left(z\right)$
is introduced as the density of the mixture, whereas for $\left(2\right)$,
we use the fact that $\mathcal{H}$ has infinite capacity, leading
to the equality $\min_{\hat{h}\in\mathcal{H}}\int f\left(\hat{h}\left(x\right)\right)q\left(x\right)dx=\int\min_{\hat{h}\in\mathcal{H}}f\left(\hat{h}\left(x\right)\right)q\left(x\right)dx$.
Moreover, for $\left(3\right)$, the property of concave function
$\phi\left(t\right)=\min_{\beta\in\mathcal{Y}_{\Delta}}\left(\sum_{i=1}^{C}\alpha_{i}l\left(\beta,i\right)t_{i}\right)$
with $t\in\mathcal{Y}_{\Delta}$ is invoked, i.e., $\mathbb{E}_{z\sim Q}\left[\phi\left(t\left(z\right)\right)\right]\leq\phi\left(\mathbb{E}_{z\sim Q}\left[t\left(z\right)\right]\right)$.

This hints us to define a non-zero divergence $D^{\alpha}$ between
multiple distributions $\mathbb{Q}_{1},...,\mathbb{Q}_{C}$ as

\[
\begin{aligned}D^{\alpha}\left(\mathbb{Q}_{1},...,\mathbb{Q}_{C}\right) & =-\min_{\hat{h}\in\mathcal{H}}\mathcal{L}_{\mathbb{Q}_{1:C}}^{\alpha}\left(\hat{h}\right)+\inf_{\beta\in\mathcal{Y}_{\Delta}}\left(\sum_{i=1}^{C}l\left(\beta,i\right)\alpha_{i}\right),\\
 & =\int-\phi\left(\left[\frac{q_{i}\left(z\right)}{q^{\alpha}\left(z\right)}\right]_{i=1}^{C}\right)q^{\alpha}\left(z\right)dz+\inf_{\beta\in\mathcal{Y}_{\Delta}}\left(\sum_{i=1}^{C}l\left(\beta,i\right)\alpha_{i}\right)
\end{aligned}
\]

which is a proper f-divergence, since $-\phi\left(t\right)$ is a
convex function, and $\inf_{\beta\in\mathcal{Y}_{\Delta}}\left(\sum_{i=1}^{C}l\left(\beta,i\right)\alpha_{i}\right)$
is just a constant. Moreover, $D^{\alpha}\left(\mathbb{Q}_{1},...,\mathbb{Q}_{C}\right)\geq0$
for all $\mathbb{Q}_{1},...,\mathbb{Q}_{C}$ and $\alpha\in\mathcal{Y}_{\Delta}$
due to the previous inequality \ref{eq:convex_inequality}. The equality
happens if there is some $\beta_{0}\in\mathcal{Y}_{\Delta}$ such
that, for all $z\in\mathcal{Z}$

\[
\beta_{0}=\argmin{\beta\in\mathcal{Y}_{\Delta}}\sum_{i=1}^{C}\alpha_{i}l\left(\beta,i\right)\frac{q_{i}\left(z\right)}{q^{\alpha}\left(z\right)}.
\]

This means $\frac{q_{i}\left(z\right)}{q^{\alpha}\left(z\right)}=A_{i},\forall i\in\left[C\right]$,
where $A_{i}$ is a constant dependent on index $i$. However, this
leads to

\[
\begin{aligned}\int q_{i}\left(z\right)dz & =A_{i}\int q^{\alpha}\left(z\right)dz\\
1 & =A_{i}
\end{aligned}
\]

i.e., $q_{i}\left(z\right)=q^{\alpha}\left(z\right),\forall i\in\left[C\right]$.
In other words, the equality happens when all distributions are the
same $\mathbb{Q}_{1}=...=\mathbb{Q}_{C}$.
\end{proof}

\subsubsection{General Domain-Invariant Representations}

For a fixed feature map $g$, the induced representation distributions
of source domains are $\mathbb{P}_{g}^{S,i}$. We then find the optimal
hypothesis $\hat{h}_{g}^{*}$ on the induced representation distributions
$\mathbb{P}_{g}^{S,i}$ by minimizing the loss 
\begin{equation}
\min_{\hat{h}\in\mathcal{H}}\sum_{i=1}^{K}\pi_{i}\mathcal{L}\left(\hat{h},h^{S,i},\mathbb{P}_{g}^{S,i}\right)=\text{min}_{\hat{h}\in\mathcal{H}}\sum_{i=1}^{K}\pi_{i}\mathcal{L}\left(\hat{h},\mathbb{D}_{g}^{S,i}\right).\label{eq:multi_loss}
\end{equation}

The general domain-invariant feature map $g^{*}$ is defined as the
one that offers the minimal optimal loss as 
\begin{equation}
\begin{aligned}g^{*}=\arg\min_{g\in\mathcal{G}}\min_{\hat{h}\in\mathcal{H}}\sum_{i=1}^{K}\pi_{i}\mathcal{L}\left(\hat{h},h^{S,i},\mathbb{P}_{g}^{i}\right)=\text{argmin}_{g\in\mathcal{G}}\text{min}_{\hat{h}\in\mathcal{H}}\sum_{i=1}^{K}\pi_{i}\mathcal{L}\left(\hat{h},\mathbb{D}_{g}^{S,i}\right)\end{aligned}
.\label{apx_eq:optimal_g_general_feature}
\end{equation}

\textit{\emph{We denote $\mathbb{P}_{g}^{s,i,c}$ as the class $c$
conditional distribution of the source domain $i$ on the latent space
and $p_{g}^{s,i,c}$ as its density function. The }}induced representation
distribution $\mathbb{P}_{g}^{S,i}$ of source domain $i$ is a mixture
of $\mathbb{P}_{g}^{s,i,c}$ as $\mathbb{P}_{g}^{S,i}=\sum_{c=1}^{C}\gamma_{i,c}\mathbb{P}_{g}^{s,i,c},$where
$\gamma_{i,c}=\mathbb{P}^{s,i}\left(y=c\right)$.

We further define $\mathbb{Q}_{g}^{s,c}:=\sum_{i=1}^{K}\frac{\pi_{i}\gamma_{i,c}}{\alpha_{c}}\mathbb{P}_{g}^{s,i,c}$
where $\alpha_{c}=\sum_{j=1}^{K}\pi_{j}\gamma_{j,c}$. Obviously,
we can interpret $\mathbb{Q}_{g}^{s,c}$ as the mixture of the \textit{\emph{class
$c$ conditional distributions of the source domains on the latent
space. The objective function in Eq. (}}\ref{eq:multi_loss}\textit{\emph{)
can be viewed as training the optimal hypothesis $\hat{h}\in\mathcal{H}$
to distinguish the samples from $\mathbb{Q}_{g}^{s,c},c\in\left[C\right]$
for a given feature map $g$. Therefore, by linking to the multi-divergence
concept developed in Theorem \ref{apx_thm:op_distance}, we achieve
the following theorem.}}
\begin{thm}
\label{apx_thm:max_distance_g}(\textbf{Theorem 6 in the main paper})
Assume that $\mathcal{H}$ has infinite capacity, we have the following
statements.

1. $D^{\alpha}\left(\mathbb{Q}_{g}^{s,1},...,\mathbb{Q}_{g}^{s,C}\right)=-\min_{\hat{h}\in\mathcal{H}}\sum_{i=1}^{K}\pi_{i}\mathcal{L}\left(\hat{h},h^{S,i},\mathbb{P}_{g}^{S,i}\right)+\text{const}$,
where $\alpha=\left[\alpha_{c}\right]_{c\in\left[C\right]}$ is defined
as above.

2. Finding the general domain-invariant feature map $g^{*}$ via the
OP in (\ref{apx_eq:optimal_g_general_feature}) is equivalent to solving
\begin{equation}
g^{*}=\argmax{g\in\mathcal{G}}\,D^{\alpha}\left(\mathbb{Q}_{g}^{s,1},...,\mathbb{Q}_{g}^{s,C}\right).\label{apx_eq:max_distance}
\end{equation}
\end{thm}

\begin{proof}
We investigate the loss on mixture

\[
\begin{aligned}\sum_{i=1}^{K}\pi_{i}\mathcal{L}\left(\hat{h},h^{S,i},\mathbb{P}_{g}^{S,i}\right) & =\sum_{i=1}^{K}\pi_{i}\sum_{c=1}^{C}\gamma_{i,c}\mathcal{L}\left(\hat{h},c,\mathbb{P}_{g}^{S,i,c}\right)\\
 & =\sum_{c=1}^{C}\alpha_{c}\mathcal{L}\left(\hat{h},c,\sum_{i=1}^{K}\frac{\pi_{i}\gamma_{i,c}}{\alpha_{c}}\mathbb{P}_{g}^{S,i,c}\right)\\
 & =\sum_{c=1}^{C}\alpha_{c}\mathcal{L}\left(\hat{h},c,\mathbb{Q}_{g}^{S,c}\right)
\end{aligned}
\]

Therefore, the loss on mixture is actually a loss on joint class-conditional
distributions $\mathbb{Q}_{g}^{S,c}=\sum_{i=1}^{K}\frac{\pi_{i}\gamma_{i,c}}{\alpha_{c}}\mathbb{P}_{g}^{S,i,c}$.
Using result from Theorem \ref{apx_thm:op_distance}, we can define
a divergence between these class-conditionals

\[
\begin{aligned}D^{\alpha}\left(\mathbb{Q}_{g}^{S,1},...,\mathbb{Q}_{g}^{S,C}\right) & =-\min_{\hat{h}\in\mathcal{H}}\sum_{c=1}^{C}\alpha_{c}\mathcal{L}\left(\hat{h},\mathbb{Q}_{g}^{S,c}\right)+\min_{\beta\in\mathcal{Y}_{\Delta}}\left(\sum_{c=1}^{C}\ell\left(\beta,i\right)\alpha_{c}\right)\\
 & =-\sum_{i=1}^{K}\pi_{i}\mathcal{L}\left(\hat{h},h^{S,i},\mathbb{P}_{g}^{S,i}\right)+\text{\text{const}}
\end{aligned}
\]
\end{proof}

\subsection{Compressed Domain-Invariant Representations}
\begin{thm}
\label{apx_thm:gap} (\textbf{Theorem 7 in the main paper}) For any
confident level $\delta\in[0,1]$ over the choice of $S$, the estimation
of loss is in the $\epsilon$-range of the true loss 
\[
\text{Pr}\left(\left|\mathcal{L}\left(\hat{h},S\right)-\mathcal{L}\left(\hat{h},\mathbb{D}_{g}^{\pi}\right)\right|\leq\epsilon\right)\geq1-\delta,
\]
where $\epsilon=\epsilon\left(\delta\right)=\left(\frac{A}{\delta}\right)^{1/2}$
is a function of $\delta$ for which $A$ is proportional to{\footnotesize{}
\begin{align*}
\frac{1}{N}\left(\sum_{i=1}^{K}\sum_{j=1}^{K}\frac{\sqrt{\pi_{i}}}{K}\mathcal{L}\left(\hat{f},\mathbb{D}^{S,j}\right)+L\sum_{i=1}^{K}\sqrt{\pi_{i}}\max_{k\in\left[K\right]}\mathbb{E}_{\mathbb{P}^{S,k}}\left[\left\Vert \Delta p^{k,i}\left(y|x\right)\right\Vert _{1}\right]+\frac{L}{K}\sum_{i=1}^{K}\sum_{j=1}^{K}\sqrt{2\pi_{i}}~d_{1/2}\left(\mathbb{P}_{g}^{S,i},\mathbb{P}_{g}^{S,j}\right)\right)^{2}.
\end{align*}
}{\footnotesize\par}
\end{thm}

\begin{proof}
\textit{}Let $S$ be a sample of $N$ data points $\left(z,y\right)\sim\mathbb{D}_{g}^{\pi}$
sampled from the mixture domain, i.e., i.e., $i\sim Cat(\pi),z\sim\mathbb{P}^{S,i}$,
and labeling with corresponding $y\sim Cat\left(\hat{h}^{S,i}\left(z\right)\right)$.
The loss of a hypothesis $h$ on a sample $\left(z,y\right)=\left(z,\hat{h}^{S,i}(z)\right)$
for some domain index $i$ is $\ell\left(\hat{h}\left(z\right),\hat{h}^{S,i}\left(z\right)\right)$.
To avoid crowded notation, we denote this loss as $\ell^{i}\left(z\right):=\ell\left(\hat{h}\left(z\right),\hat{h}^{S,i}\left(z\right)\right)$.

Let $N=\sum_{i=1}^{K}N_{i}$, where each $N_{i}$ is the number of
sample drawn from domain $i$. The estimation of loss on a particular
domain $i$ is

\[
\mathcal{L}\left(\hat{h},S^{i}\right)=\sum_{j=1}^{N_{i}}\frac{1}{N_{i}}\ell^{i}\left(z_{j}\right).
\]

This estimation is unbiased estimation, i.e., $\mathbb{E}_{S^{i}\sim\left(\mathbb{D}_{g}^{S,i}\right)^{N_{i}}}\left[\mathcal{L}\left(\hat{h},S^{i}\right)\right]=\mathbb{E}_{z\sim\mathbb{P}_{g}^{S,i}}\left[\ell^{i}\left(z\right)\right]=\mathcal{L}\left(\hat{h},\mathbb{D}_{g}^{S,i}\right)$.
Furthermore, loss estimation on source mixture is 
\[
\begin{aligned}\mathcal{L}\left(h,S\right) & =\sum_{i=1}^{K}\sum_{j=1}^{N_{i}}\frac{1}{N}\ell^{i}\left(z_{j}\right)\end{aligned}
\]

This estimation is also an unbiased estimation 
\[
\begin{aligned}\mathbb{E}_{S\sim\left(\mathbb{D}_{g}^{\pi}\right)^{N}}\left[\mathcal{L}\left(\hat{h},S\right)\right]= & \mathbb{E}_{\left\{ N_{i}\right\} }\left[\mathbb{E}_{S^{i}}\left[\mathcal{L}\left(\hat{h},S\right)\right]\right]\\
= & \sum_{i=1}^{K}\pi_{i}\mathcal{L}\left(\hat{h},\mathbb{D}_{g}^{S,i}\right)=\mathcal{L}\left(\hat{h},\mathbb{D}_{g}^{\pi}\right).
\end{aligned}
\]

Therefore, we can bound the concentration of $\mathcal{L}\left(\hat{h},S\right)$
around its mean value $\mathcal{L}\left(\hat{h},\mathbb{D}_{g}^{\pi}\right)$
using Chebyshev's inequality 
\[
\begin{aligned}\text{Pr}\left(\left|\mathcal{L}\left(\hat{h},S\right)-\mathcal{L}\left(\hat{h},\mathbb{D}_{g}^{\pi}\right)\right|\leq\epsilon\right) & \geq1-\frac{\text{Var}_{S\sim\left(\mathbb{D}_{g}^{\pi}\right)^{N}}\left[\mathcal{L}\left(\hat{h},S\right)\right]}{\epsilon^{2}}\end{aligned}
\]

which is equivalent to

\[
\begin{aligned}\text{Pr}\left(\left|\mathcal{L}\left(\hat{h},S\right)-\mathcal{L}\left(\hat{h},\mathbb{D}_{g}^{\pi}\right)\right|\leq\sqrt{\frac{\text{Var}_{S\sim\left(\mathbb{D}_{g}^{\pi}\right)^{N}}\left[\mathcal{L}\left(\hat{h},S\right)\right]}{\delta}}\right) & \geq1-\delta\end{aligned}
\]

The variance of $\mathcal{L}\left(\hat{h},S\right)$ is

\begin{equation}
\begin{aligned}\text{Var}_{S\sim\left(\mathbb{D}_{g}^{\pi}\right)^{N}}\left[\mathcal{L}\left(\hat{h},S\right)\right] & \stackrel{\left(1\right)}{=}\frac{1}{N}\text{Var}_{\left(z,y\right)\sim\mathbb{D}_{g}^{\pi}}\left[\ell\left(\hat{h}(z),y\right)\right]\\
 & \stackrel{\left(2\right)}{=}\frac{1}{N}\sum_{i=1}^{K}\pi_{i}\left(\text{Var}_{z\sim\mathbb{P}_{g}^{S,i}}\left[\ell^{i}\left(z\right)\right]+\mathbb{E}_{z\sim\mathbb{P}_{g}^{S,i}}\left[\ell^{i}\left(z\right)\right]^{2}\right)-\left(\mathbb{E}_{\left(z,y\right)\sim\mathbb{D}_{g}^{\pi}}\left[\ell\left(\hat{h}(z),y\right)\right]\right)^{2}\\
 & \leq\frac{1}{N}\sum_{i=1}^{K}\pi_{i}\left(\text{Var}_{z\sim\mathbb{P}_{g}^{S,i}}\left[\ell^{i}\left(z\right)\right]+\mathcal{L}\left(\hat{h},\mathbb{D}_{g}^{S,i}\right)^{2}\right)\\
 & \leq\frac{1}{N}\sum_{i=1}^{K}\pi_{i}\text{Var}_{z\sim\mathbb{P}_{g}^{S,i}}\left[\ell^{i}\left(z\right)\right]+\frac{1}{N}\left(\sum_{i=1}^{K}\sqrt{\pi_{i}}\mathcal{L}\left(\hat{h},\mathbb{D}_{g}^{S,i}\right)\right)^{2}
\end{aligned}
\label{eq:var_Lhs}
\end{equation}

$\stackrel{\left(1\right)}{=}$ is true since $\mathcal{L}\left(h,S\right)$
is the sum of $N$ i.i.d. random variable $\ell\left(h(z),y\right)$
with $\left(z,y\right)$ sampled from the same distribution $\mathbb{D}_{g}^{\pi}$.
In $\left(2\right)$, the variance of w.r.t. a distribution mixture
is related to mean and variance of constituting distribution, i.e.,
$\text{Var}_{\sum_{i}\pi_{i}\mathbb{P}_{i}}\left[X\right]=\sum_{i}\pi_{i}\left(\text{Var}_{\mathbb{P}_{i}}\left[X\right]+\mathbb{E}_{\mathbb{P}_{i}}\left[X\right]^{2}\right)-\mathbb{E}_{\sum_{i}\pi_{i}\mathbb{P}_{i}}\left[X\right]^{2}$.

We reuse the result of \ref{eq:latent_bound} in Theorem \ref{apx_thm:general_dom_inv-1},
substituting $\mathbb{D}_{g}^{T}\equiv\mathbb{D}_{g}^{S,i}$, $\mathbb{D}_{g}^{\pi}\equiv\mathbb{D}_{g}^{S,j}$
to obtain

\[
\begin{aligned}\mathcal{L}\left(\hat{h},\mathbb{D}_{g}^{S,i}\right) & \leq\mathcal{L}\left(\hat{f},\mathbb{D}^{S,j}\right)+L\mathbb{E}_{\mathbb{P}^{S,i}}\left[\|\Delta p^{i,j}(y|x)\|_{1}\right]+L\sqrt{2}\:d_{1/2}\left(\mathbb{P}_{g}^{S,i},\mathbb{P}_{g}^{S,j}\right)\\
 & \leq\frac{1}{K}\sum_{j=1}^{K}\left(\mathcal{L}\left(\hat{f},\mathbb{D}^{S,j}\right)+L\max_{k\in\left[K\right]}\mathbb{E}_{\mathbb{P}^{S,k}}\left[\left\Vert \Delta p^{k,i}\left(y|x\right)\right\Vert _{1}\right]+L\sqrt{2}~d_{1/2}\left(\mathbb{P}_{g}^{S,i},\mathbb{P}_{g}^{S,j}\right)\right).
\end{aligned}
\]

Therefore, the right hand side of \ref{eq:var_Lhs} is upper by $A$,
where $A$ is

{\footnotesize{}
\begin{align*}
A & =\frac{1}{N}\sum_{i=1}^{K}\pi_{i}\text{Var}_{z\sim\mathbb{P}_{g}^{S,i}}\left[\ell^{i}\left(z\right)\right]\\
 & +\frac{1}{N}\left(\sum_{i=1}^{K}\sum_{j=1}^{K}\frac{\sqrt{\pi_{i}}}{K}\mathcal{L}\left(\hat{f},\mathbb{D}^{S,j}\right)+L\sum_{i=1}^{K}\sqrt{\pi_{i}}\max_{k\in\left[K\right]}\mathbb{E}_{\mathbb{P}^{S,k}}\left[\left\Vert \Delta p^{k,i}\left(y|x\right)\right\Vert _{1}\right]+\frac{L}{K}\sum_{i=1}^{K}\sum_{j=1}^{K}\sqrt{2\pi_{i}}~d_{1/2}\left(\mathbb{P}_{g}^{S,i},\mathbb{P}_{g}^{S,j}\right)\right)^{2}.
\end{align*}
}{\footnotesize\par}

This concludes our proof, where the concentration inequality is

\[
\begin{aligned}\text{Pr}\left(\left|\mathcal{L}\left(\hat{h},S\right)-\mathcal{L}\left(\hat{h},\mathbb{D}_{g}^{\pi}\right)\right|\leq\sqrt{\frac{A}{\delta}}\right) & \geq1-\delta\end{aligned}
.
\]
\end{proof}

\section{Appendix C: Trade-Off in Learning DI Representations}

\begin{lem}
\label{lem:Hellinger-loss}Given a labeling function $f:\mathcal{X}\goto\mathcal{Y}_{\simplex}$
and a hypothesis $\hat{f}:\mathcal{X}\goto\mathcal{Y}_{\simplex}$,
let denote $\mathbb{P}_{\mathcal{Y}}^{f}$ and $\mathbb{P}_{\mathcal{Y}}^{\hat{f}}$
as two label marginal distributions induced by $f$ and $\hat{f}$
on the data distribution $\mathbb{P}$. Particularly, to sample $y\sim\mathbb{P}_{\mathcal{Y}}^{f}$,
we first sample $x\sim\mathbb{P}$ (i.e., $\mathbb{P}$ is the data
distribution with the density function $p$) and then sample $y\sim Cat\left(f\left(x\right)\right)$,
while similar to sample $y\sim\mathbb{P}_{\mathcal{Y}}^{\hat{f}}$.
We then have 
\[
d_{1/2}\left(\mathbb{P}_{\mathcal{Y}}^{f},\mathbb{P}_{\mathcal{Y}}^{\hat{f}}\right)\leq\mathcal{L}\left(\hat{f},f,\mathbb{P}\right)^{1/2},
\]
where the loss $\mathcal{L}$ is defined based on the Hellinger loss
$\ell\left(\hat{f}\left(x\right),f\left(x\right)\right)=D_{1/2}\left(\hat{f}\left(x\right),f\left(x\right)\right)=2\sum_{i=1}^{C}\left[\sqrt{\hat{f}\left(x,i\right)}-\sqrt{f\left(x,i\right)}\right]^{2}$.
\end{lem}

\begin{proof}
We have
\begin{align*}
D_{1/2}\left(\mathbb{P}_{\mathcal{Y}}^{f},\mathbb{P}_{\mathcal{Y}}^{\hat{f}}\right) & =2\sum_{i=1}^{C}\left(\sqrt{p^{f}\left(y\right)}-\sqrt{p^{\hat{f}}\left(y\right)}\right)^{2}\\
= & 2\sum_{i=1}^{C}\left(\sqrt{\int p^{f}\left(y=i\mid x\right)p(x)dx}-\sqrt{\int p^{\hat{f}}\left(y=i\mid x\right)p(x)dx}\right)^{2}\\
= & 2\sum_{i=1}^{C}\left(\sqrt{\int f\left(x,i\right)p(x)dx}-\sqrt{\int\hat{f}\left(x,i\right)p(x)dx}\right)^{2}\\
= & 2\sum_{i=1}^{C}\left[\int f\left(x,i\right)p(x)dx+\int\hat{f}\left(x,i\right)p(x)dx-2\sqrt{\int f\left(x,i\right)p(x)dx}\sqrt{\int\hat{f}\left(x,i\right)p(x)dx}\right]\\
\overset{(1)}{\leq} & 2\sum_{i=1}^{C}\left[\int f\left(x,i\right)p(x)dx+\int\hat{f}\left(x,i\right)p(x)dx-2\sqrt{\int f\left(x,i\right)\hat{f}\left(x,i\right)}p(x)dx\right]\\
= & 2\sum_{i=1}^{C}\int\left[\sqrt{f\left(x,i\right)}-\sqrt{\hat{f}\left(x,i\right)}\right]^{2}p(x)dx=\int2\sum_{i=1}^{C}\left[\sqrt{f\left(x,i\right)}-\sqrt{\hat{f}\left(x,i\right)}\right]^{2}p(x)dx\\
= & \int D_{1/2}\left(\hat{f}\left(x\right),f\left(x\right)\right)p(x)dx=\mathcal{L}\left(\hat{f},f,\mathbb{P}\right),
\end{align*}
where we note that in the derivation in $\overset{(1)}{\leq}$, we
use Cauchy-Schwarz inequality: $\int\hat{f}\left(x,i\right)p\left(x\right)dx\int f\left(x,i\right)p\left(x\right)dx\geq\left(\int\sqrt{\hat{f}\left(x,i\right)f\left(x,i\right)}p\left(x\right)dx\right)^{2}$.

Therefore, we reach the conclusion as 
\[
d_{1/2}\left(\mathbb{P}_{\mathcal{Y}}^{f},\mathbb{P}_{\mathcal{Y}}^{\hat{f}}\right)\leq\mathcal{L}\left(\hat{f},f,\mathbb{P}\right)^{1/2}.
\]
\end{proof}
\begin{lem}
\label{lem:key_tradeoff}Consider the hypothesis $\hat{f}=\hat{h}\circ g$.
We have the following inequalities w.r.t. the source and target domains:

(i) $d_{1/2}\left(\mathbb{\hat{P}}_{\mathcal{Y}}^{T},\mathbb{P}_{\mathcal{Y}}^{T}\right)\leq\mathcal{L}\left(\hat{h}\circ g,f^{T},\mathbb{P}^{T}\right)^{1/2},$
where $\mathbb{P}_{\mathcal{Y}}^{T}$ is the label marginal distribution
induced by $f^{T}$ on $\mathbb{P}^{T}$, while $\hat{\mathbb{P}}_{\mathcal{Y}}^{T}$
is the label marginal distribution induced by $\hat{f}$ on $\mathbb{P}^{T}$.

(ii) $d_{1/2}\left(\mathbb{P}_{\mathcal{Y}}^{\pi},\hat{\mathbb{P}}_{\mathcal{Y}}^{\pi}\right)\leq\left[\sum_{i=1}^{K}\pi_{i}\mathcal{L}\left(\hat{h}\circ g,f^{S,i},\mathbb{P}^{S,i}\right)\right]^{1/2}$,
where $\mathbb{P}_{\mathcal{Y}}^{\pi}:=\sum_{i=1}^{K}\pi_{i}\mathbb{P}_{\mathcal{Y}}^{S,i}$
with $\mathbb{P}_{\mathcal{Y}}^{S,i}$ to be induced by $f^{S,i}$
on $\mathbb{P}^{S,i}$ and $\hat{\mathbb{P}}_{\mathcal{Y}}^{\pi}:=\sum_{i=1}^{K}\pi_{i}\hat{\mathbb{P}}_{\mathcal{Y}}^{S,i}$
with $\hat{\mathbb{P}}_{\mathcal{Y}}^{S,i}$to be induced by $\hat{f}$
on $\mathbb{P}^{S,i}$ (i.e., equivalently, the label marginal distribution
induced by $\hat{f}$ on $\mathbb{P}^{\pi}:=\sum_{i=1}^{K}\pi_{i}\mathbb{P}^{S,i}$).
\end{lem}

\begin{proof}
(i) The proof of this part is obvious from Lemma \ref{lem:Hellinger-loss}
by considering $f^{T}$ as $f$ and $\mathbb{P}^{T}$ as $\mathbb{P}$.

(ii) By the convexity of $D_{1/2}$, which is a member of $f$-divergence
family, we have
\begin{align*}
D_{1/2}\left(\mathbb{P}_{\mathcal{Y}}^{\pi},\hat{\mathbb{P}}_{\mathcal{Y}}^{\pi}\right) & =D_{1/2}\left(\sum_{i=1}^{K}\pi_{i}\mathbb{P}_{\mathcal{Y}}^{S,i},\sum_{i=1}^{K}\pi_{i}\hat{\mathbb{P}}_{\mathcal{Y}}^{S,i}\right)\leq\sum_{i=1}^{K}\pi_{i}D_{1/2}\left(\mathbb{P}_{\mathcal{Y}}^{S,i},\hat{\mathbb{P}}_{\mathcal{Y}}^{S,i}\right)\\
\overset{(1)}{\leq} & \sum_{i=1}^{K}\pi_{i}\mathcal{L}\left(\hat{h}\circ g,f^{S,i},\mathbb{P}^{S,i}\right),
\end{align*}
where the derivation in $\overset{(1)}{\leq}$ is from Lemma \ref{lem:Hellinger-loss}.
Therefore, we reach the conclusion as
\[
d_{1/2}\left(\mathbb{P}_{\mathcal{Y}}^{\pi},\hat{\mathbb{P}}_{\mathcal{Y}}^{\pi}\right)\leq\left[\sum_{i=1}^{K}\pi_{i}\mathcal{L}\left(\hat{h}\circ g,f^{S,i},\mathbb{P}^{S,i}\right)\right]^{1/2}.
\]
\end{proof}
\begin{thm}
\label{apx_thm:trade_off_bounds} (\textbf{Theorem 8 in the main paper})
Consider a feature extractor $g$ and a hypothesis $\hat{h}$, the
Hellinger distance between two label marginal distributions $\mathbb{P}_{\mathcal{Y}}^{\pi}$
and $\mathbb{P}_{\mathcal{Y}}^{T}$ can be upper-bounded as: 

(i) $d_{1/2}\left(\mathbb{P}_{\mathcal{Y}}^{\pi},\mathbb{P}_{\mathcal{Y}}^{T}\right)\leq\left[\sum_{k=1}^{K}\pi_{k}\mathcal{L}\left(\hat{h}\circ g,f^{S,k},\mathbb{P}^{S,k}\right)\right]^{1/2}+d_{1/2}\left(\mathbb{P}_{g}^{T},\mathbb{P}_{g}^{\pi}\right)+\mathcal{L}\left(\hat{h}\circ g,f^{T},\mathbb{P}^{T}\right)^{1/2}.$

(ii) $d_{1/2}\left(\mathbb{P}_{\mathcal{Y}}^{\pi},\mathbb{P}_{\mathcal{Y}}^{T}\right)\leq\left[\sum_{i=1}^{K}\pi_{i}\mathcal{L}\left(\hat{h}\circ g,f^{S,i},\mathbb{P}^{S,i}\right)\right]^{1/2}+\sum_{i=1}^{K}\sum_{j=1}^{K}\frac{\sqrt{\pi_{j}}}{K}d_{1/2}\left(\mathbb{P}_{g}^{S,i},\mathbb{P}_{g}^{S,j}\right)+\sum_{i=1}^{K}\sum_{j=1}^{K}\frac{\sqrt{\pi_{j}}}{K}d_{1/2}\left(\mathbb{P}_{g}^{T},\mathbb{P}_{g}^{S,i}\right)+\mathcal{L}\left(\hat{h}\circ g,f^{T},\mathbb{P}^{T}\right)^{1/2}.$

Here we note that the general loss $\mathcal{L}$ is defined based
on the Hellinger loss $\ell$ defined as $\ell\left(\hat{f}\left(x\right),f(x)\right)=D_{1/2}\left(\hat{f}\left(x\right),f(x)\right)$.
\end{thm}

\begin{proof}
(i)We define $\mathbb{P}_{\mathcal{Y}}^{\pi},\hat{\mathbb{P}}_{\mathcal{Y}}^{\pi}$
and $\mathbb{P}_{\mathcal{Y}}^{T},\hat{\mathbb{P}}_{\mathcal{Y}}^{T}$
as in Lemma \ref{lem:key_tradeoff}. Recap that to sample $y\sim\hat{\mathbb{P}}_{\mathcal{Y}}^{\pi}$,
we sample $k\sim Cat(\pi),x\sim\mathbb{P}^{S,k}$ (i.e., $x\sim\mathbb{P}^{\pi}:=\sum_{k=1}^{K}\pi_{k}\mathbb{P}^{S,k}$),
compute $z=g\left(x\right)$ (i.e., $z\sim\mathbb{P}_{g}^{\pi}$)
, and $y\sim Cat\left(\hat{h}\left(z\right)\right)$, while similar
to draw $y\sim\mathbb{\hat{P}}_{\mathcal{Y}}^{T}$.

Using triangle inequality for Hellinger distance, we have

\[
\begin{aligned}d_{1/2}\left(\mathbb{P}_{\mathcal{Y}}^{\pi},\mathbb{P}_{\mathcal{Y}}^{T}\right) & \leq d_{1/2}\left(\mathbb{P}_{\mathcal{Y}}^{\pi},\hat{\mathbb{P}}_{\mathcal{Y}}^{\pi}\right)+d_{1/2}\left(\mathbb{\hat{P}}_{\mathcal{Y}}^{\pi},\hat{\mathbb{P}}_{\mathcal{Y}}^{T}\right)+d_{1/2}\left(\hat{\mathbb{P}}_{\mathcal{Y}}^{T},\mathbb{P}_{\mathcal{Y}}^{T}\right)\end{aligned}
.
\]

Referring to Lemma \ref{lem:key_tradeoff}, we achieve

\[
\begin{aligned}d_{1/2}\left(\mathbb{P}_{\mathcal{Y}}^{\pi},\mathbb{P}_{\mathcal{Y}}^{T}\right) & \leq\left[\sum_{i=1}^{K}\pi_{i}\mathcal{L}\left(\hat{h}\circ g,f^{S,i},\mathbb{P}^{S,i}\right)\right]^{1/2}+d_{1/2}\left(\mathbb{P}_{\hat{\mathcal{Y}}}^{\pi},\mathbb{P}_{\hat{\mathcal{Y}}}^{T}\right)+\mathcal{L}\left(\hat{h}\circ g,f^{T},\mathbb{P}^{T}\right)^{1/2}.\end{aligned}
\]

From the monotonicity of Hellinger distance, when applying to $\mathbb{P}_{g}^{T}$
and $\mathbb{P}_{g}^{\pi}$ with the same transition probability $p\left(y=i\mid z\right)=\hat{h}\left(z,i\right)$
for obtaining $\mathbb{P}_{\hat{\mathcal{Y}}}^{\pi}\text{ and }\mathbb{P}_{\hat{\mathcal{Y}}}^{T}$,
we have

\[
d_{1/2}\left(\mathbb{P}_{\hat{\mathcal{Y}}}^{\pi},\mathbb{P}_{\hat{\mathcal{Y}}}^{T}\right)\leq d_{1/2}\left(\mathbb{P}_{g}^{\pi},\mathbb{P}_{g}^{T}\right).
\]

Finally, we reach the conclusion as

\[
\begin{aligned}d_{1/2}\left(\mathbb{P}_{\mathcal{Y}}^{\pi},\mathbb{P}_{\mathcal{Y}}^{T}\right) & \leq\left[\sum_{i=1}^{K}\pi_{i}\mathcal{L}\left(\hat{h}\circ g,f^{S,i},\mathbb{P}^{S,i}\right)\right]^{1/2}+d_{1/2}\left(\mathbb{P}_{g}^{\pi},\mathbb{P}_{g}^{T}\right)+\mathcal{L}\left(\hat{h}\circ g,f^{T},\mathbb{P}^{T}\right)^{1/2}\end{aligned}
.
\]

(ii) From Lemma \ref{lem:decompose-mixture-distance}, we can decompose
the data shift term and use triangle inequality again, hence arriving
at

\[
\begin{aligned}d_{1/2}\left(\mathbb{P}_{\mathcal{Y}}^{\pi},\mathbb{P}_{\mathcal{Y}}^{T}\right) & \leq\left[\sum_{i=1}^{K}\pi_{i}\mathcal{L}\left(\hat{h}\circ g,f^{S,i},\mathbb{P}^{S,i}\right)\right]^{1/2}+\sum_{j=1}^{K}\sqrt{\pi_{j}}d_{1/2}\left(\mathbb{P}_{g}^{T},\mathbb{P}_{g}^{S,j}\right)+\mathcal{L}\left(\hat{h}\circ g,f^{T},\mathbb{P}^{T}\right)^{1/2}\\
 & \leq\left[\sum_{i=1}^{K}\pi_{i}\mathcal{L}\left(\hat{h}\circ g,f^{S,i},\mathbb{P}^{S,i}\right)\right]^{1/2}+\sum_{i=1}^{K}\sum_{j=1}^{K}\frac{\sqrt{\pi_{j}}}{K}\left(d_{1/2}\left(\mathbb{P}_{g}^{T},\mathbb{P}_{g}^{S,i}\right)+d_{1/2}\left(\mathbb{P}_{g}^{S,i},\mathbb{P}_{g}^{S,j}\right)\right)\\
 & +\mathcal{L}\left(\hat{h}\circ g,f^{T},\mathbb{P}^{T}\right)^{1/2}.
\end{aligned}
\]

This concludes our proof.
\end{proof}
The loss $\mathcal{L}$ in Theorem \ref{apx_thm:trade_off_bounds}
defined based on the Hellinger loss $\ell$ defined as $\ell\left(\hat{f}\left(x\right),f(x)\right)=D_{1/2}\left(\hat{f}\left(x\right),f(x)\right)$,
while theory development in previous sections bases on the loss $\ell$
which has the specific form
\begin{equation}
\ell\left(\hat{f}\left(x\right),f\left(x\right)\right)=\sum_{i=1}^{C}l\left(\hat{f}\left(x\right),i\right)f\left(x,i\right).\label{eq:loss_family}
\end{equation}

To make it more consistent, we discuss under which condition the Hellinger
loss is in the family defined in (\ref{eq:loss_family}). It is evident
that if the labeling function $f$ satisfying $f\left(x,i\right)>0,\forall x\sim\mathbb{P}$
and $i\in\left[C\right]$, for example, we apply label smoothing \citep{Szegedy_cvpr_label_smoothing,Muller_nips2019_when_label_smoothing}
on ground-truth labels, the Hellinger loss is in the family of interest.
That is because the following derivation:
\begin{align*}
D_{1/2}\left(\hat{f}\left(x\right),f(x)\right) & =2\sum_{i=1}^{C}\left[\sqrt{\hat{f}\left(x,i\right)}-\sqrt{f\left(x,i\right)}\right]^{2}\\
= & 2\sum_{i=1}^{C}\left[\sqrt{\frac{\hat{f}\left(x,i\right)}{f\left(x,i\right)}}-1\right]^{2}f\left(x,i\right),
\end{align*}
where we consider $l\left(\hat{f}\left(x\right),i\right)=\left[\sqrt{\frac{\hat{f}\left(x,i\right)}{f\left(x,i\right)}}-1\right]^{2}$.

\section{Appendix D: Additional Experiments}

\subsection{Experiment on Colored MNIST}

\subsubsection{Dataset}

We conduct experiments on the colored MNIST dataset \citep{OOD_IRM_arjovsky_2020}
whose data is generated as follow. Firstly, for any original image
$X$ in the MNIST dataset \citep{lecun2010mnist}, the value of digit
feature is $Z_{d}=0$ if the image's digit is from $0\goto4$, while
$Z_{d}=1$ is assigned to image with digit from $5\goto9$. Next,
the ground-truth label for the image $X$ is also binary and sampled
from either $\mathbb{P}(Y|Z_{d}=1)$ or $\mathbb{P}(Y|Z_{d}=0)$,
depending on the value of digit feature $Z_{d}$. These binomial distributions
are such that $\mathbb{P}(Y=1|Z_{d}=1)=\mathbb{P}(Y=0|Z_{d}=0)=0.75$.
Next, the color feature binary random variable $Z_{c}$ is assigned
to each image conditioning on its label, i.e., $z_{C}\sim\mathbb{P}(Z_{C}|Y=1)$
or $z_{c}\sim\mathbb{P}(Z_{C}|Y=0)$ with $\mathbb{P}(Z_{c}=1|Y=1)=\mathbb{P}(Z_{c}=0|Y=0)=\theta$,
depending on the domain. Finally, we color the image red if $Z_{c}=0$
or green if $Z_{c}=1$.

For both DG and MSDA experiments, there are 7 source domains generated
by setting $\mathbb{P}(Z_{c}=1|Y=1)=\mathbb{P}(Z_{c}=0|Y=0)=\theta^{S,i}$
where $\theta^{s,i}\sim Uni\left(\left[0.6,1\right]\right)$ for $i=1,\ldots,7$.
In our actual implemetation, we take $\theta^{s,i}=0.6+\text{\ensuremath{\frac{0.4}{7}\left(i-1\right)}}$.
The two target domains are created with $\theta^{T,i}\in\text{\ensuremath{\left\{ 0.05,0.7\right\} }}$
for $i=1,2$. After domain creation, data from each domain is split
into training set and validation set. For DG experiment, no data from
the target domain is used in training. On the other hand, the same
train-validation split is applied to target domains in MSDA and the
unlabeled training splits are used for training, while the validation
splits are used for testing.

\begin{figure}
\begin{centering}
\includegraphics[width=0.6\textwidth]{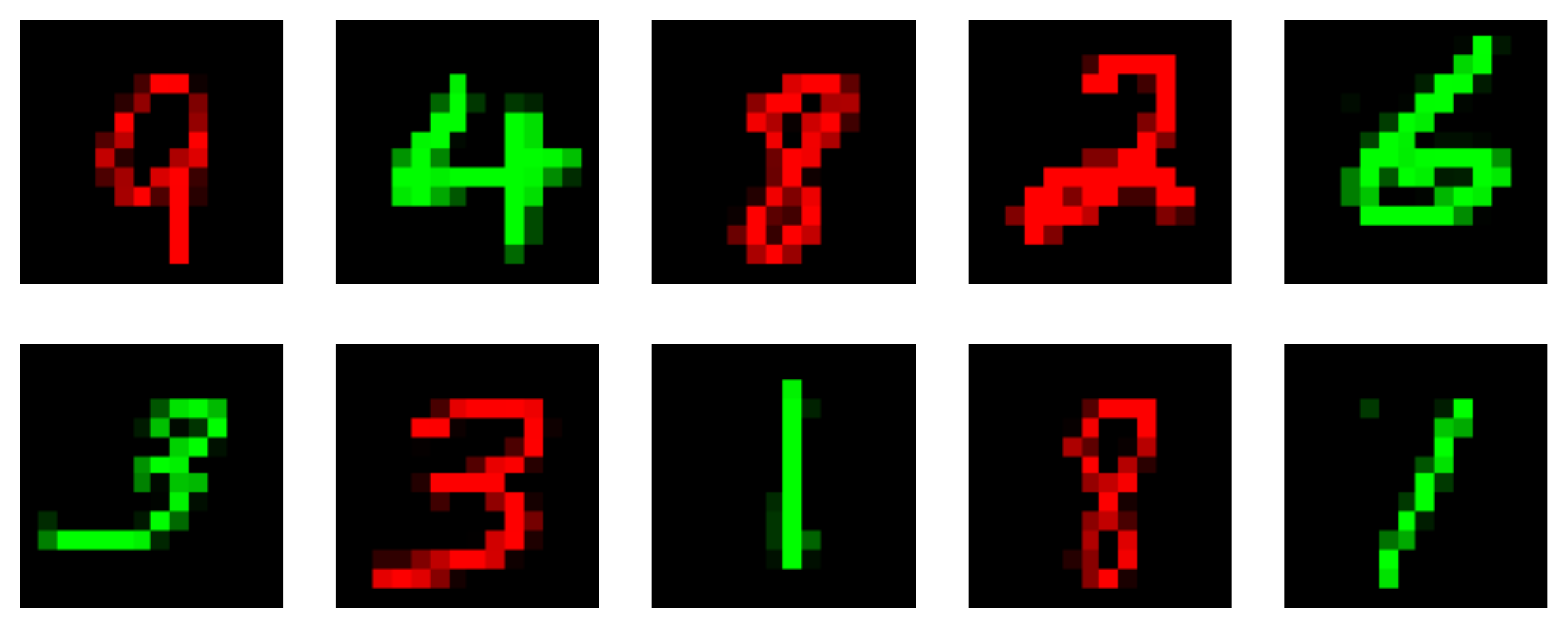}
\par\end{centering}
\caption{\label{fig:ColoredMNIST-illus}Images in Colored MNIST dataset are
\textquotedblleft colored\textquotedblright{} according to color feature
$Z_{C}$. If $Z_{C}=0$, channel $0$ is kept while channel $1$ contains
all $0$, corresponding to red images. Similarly, $Z_{C}=1$ means
channel $1$ is kept intact while channel $0$ is zero-out, represented
by green images.}

\end{figure}

\subsubsection{Model}

We train a hypothesis $\hat{f}=\hat{h}\circ g$ and minimize the classification
loss w.r.t. entire source data

\[
\mathcal{L}_{gen}:=\sum_{i=1}^{7}\frac{N_{i}}{N_{S}}\mathbb{E}_{\left(x,y\right)\sim\mathbb{D}^{S,i}}\left[CE\left(\hat{h}\left(g\left(x\right)\right),y\right)\right],
\]
where CE is the cross-entropy loss, $N_{i}$ is the number of samples
from domain $i$, and $N_{S}$is the total number of source samples.

To align source-source representation distribution, we apply adversarial
learning \citep{GAN_Goodfellow} as in \citep{DA_ganin_DANN_2016},
in which a min-max game is played, where the domain discriminator
$\hat{h}^{s-s}$ tries to predict domain labels from input representations,
while the feature extractor (generator) $g$ tries to fool the domain
discriminator, i.e., $\min_{g}\max_{\hat{h}^{s-s}}\mathcal{L}_{disc}^{s-s}$.
The source-source compression loss is defined as

\[
\mathcal{L}_{disc}^{s-s}:=\sum_{i=1}^{7}\frac{N_{i}}{N_{S}}\mathbb{E}_{x\sim\mathbb{P}^{S,i}}\left[-CE\left(\log\hat{h}^{s-s}\left(g\left(x\right)\right),i\right)\right],
\]

where $i$ is the domain label. It is well-known \citep{GAN_Goodfellow}
that if we search $\hat{h}^{s-s}$ in a family with infinite capacity
then 
\[
\max_{\hat{h}^{s-s}}\mathcal{L}_{disc}^{s-s}=JS\left(\mathbb{P}_{g}^{S,1},...,\mathbb{P}_{g}^{S,7}\right).
\]
.

Similarly, alignment between source and target feature distribution
is enforced by employing adversarial learning between another discriminator
$\hat{h}^{s-t}$ and the encoder $g$, with the objective is $\min_{g}\max_{\hat{h}^{s-s}}\mathcal{L}_{disc}^{s-t}$.
The loss function for source-target compression is

\[
\mathcal{L}_{disc}^{s-t}:=\frac{N_{k}}{N_{S}+N_{T}}\mathbb{E}_{x\sim\mathbb{P}^{\pi,S}}\left[-\log\hat{h}^{s-t}\left(g\left(x\right)\right)\right]+\frac{N_{T}}{N_{S}+N_{T}}\mathbb{E}_{x\sim\mathbb{P}^{T}}\left[-\log\left(1-\hat{h}^{s-t}\left(g\left(x\right)\right)\right)\right],
\]

where $\mathbb{P}^{\pi,S}=\sum_{i=1}^{7}\frac{N_{i}}{N_{S}}\mathbb{P}^{S,i}$
is the source mixture and $\mathbb{P}^{T}$ is the chosen target domain
among the two.

Finally, for DG we optimize the objective

\begin{equation}
\min_{g}\left(\min_{\hat{h}}\mathcal{L}_{gen}+\lambda\max_{\hat{h}^{s-s}}\mathcal{L}_{disc}^{s-s}\right),\label{eq:exp_dg_loss}
\end{equation}

where $\lambda$ being the trade-off hyperparameter: $\lambda=0$
corresponds to DG's general DI representation, while larger $\lambda$
corresponds to more compressed DI representation. On the other hand,
the objective for MSDA setting is

\begin{equation}
\min_{g}\left(\min_{\hat{h}}\mathcal{L}_{gen}+\lambda^{s-s}\max_{\hat{h}^{s-s}}\mathcal{L}_{com}^{s-s}+\lambda^{s-t}\max_{\hat{h}^{s-t}}\mathcal{L}_{com}^{s-t}\right),\label{eq:exp_msda_loss}
\end{equation}

with $\lambda^{s-s}$ controls the source-source compression and $\lambda^{s-t}$
controls the source-target compression.

Our implementation is based largely on Domain Bed repository \citep{domainbed}.
Specifically, the encoder $g$ is a convolutional neural network with
4 cnn layers, each is accompanied by RELU activation and batchnorm,
while the classifier $\hat{h}$ and discriminators $\hat{h}^{s-s},\hat{h}^{s-t}$
are densely connected multi-layer perceptions with 3 layers. Our code
can be found in the zip file accompanying this appendix. Moreover,
our experiments were run on one Tesla V100 GPU and it took around
30 minutes for one training on Colored MNIST.

\subsubsection{Multiple-source Domain Adaptation}

In additional to DG experiment provided in Section 3 in the main paper,
further experiment is conducted on MSDA, in which source-target compression
is applied in additional to source-source compression. Specifically,
unlabeled data from a target domain is supplied for training, whose
label is $1$ while all labeled source data has label $0$, and the
source-target discriminator is tasked with classifying them. We experiment
with 2 target domains $\theta^{T,i}\in\left\{ 0.5,0.7\right\} $ separately
and report accuracy on target domain for different compression strength.
The result is presented in Figure \ref{fig:Accuracy-target-msda}.

\begin{figure}
\begin{centering}
\subfloat[Distant target domain $\theta=0.05$]{\begin{centering}
\includegraphics[width=0.3\textwidth]{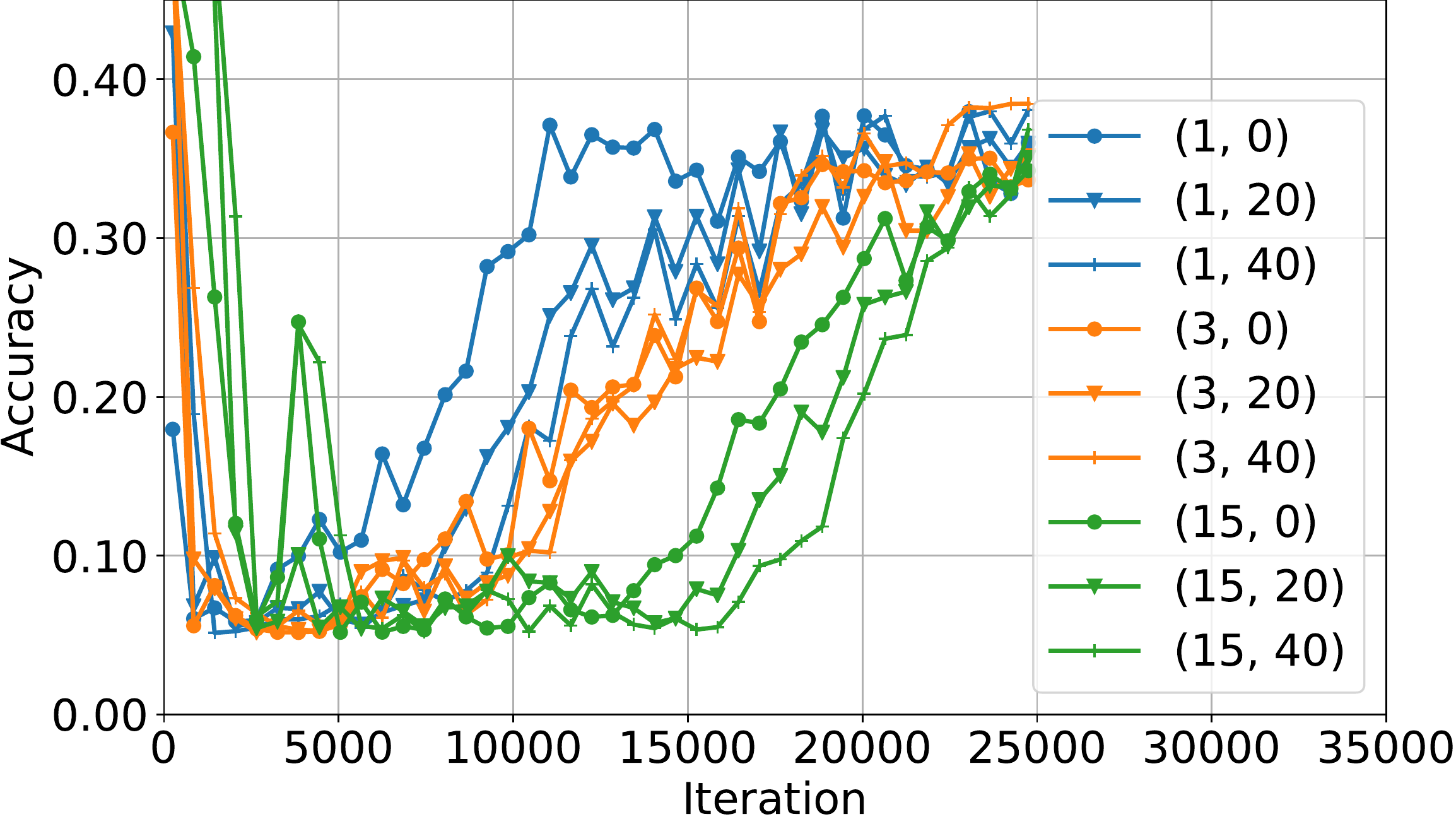}
\par\end{centering}
}\hfill{}\subfloat[\label{fig:MSDA-close-src-tar}Close target domain $\theta=0.7$,
only $\lambda^{s-t}$ changes]{\begin{centering}
\includegraphics[width=0.3\textwidth]{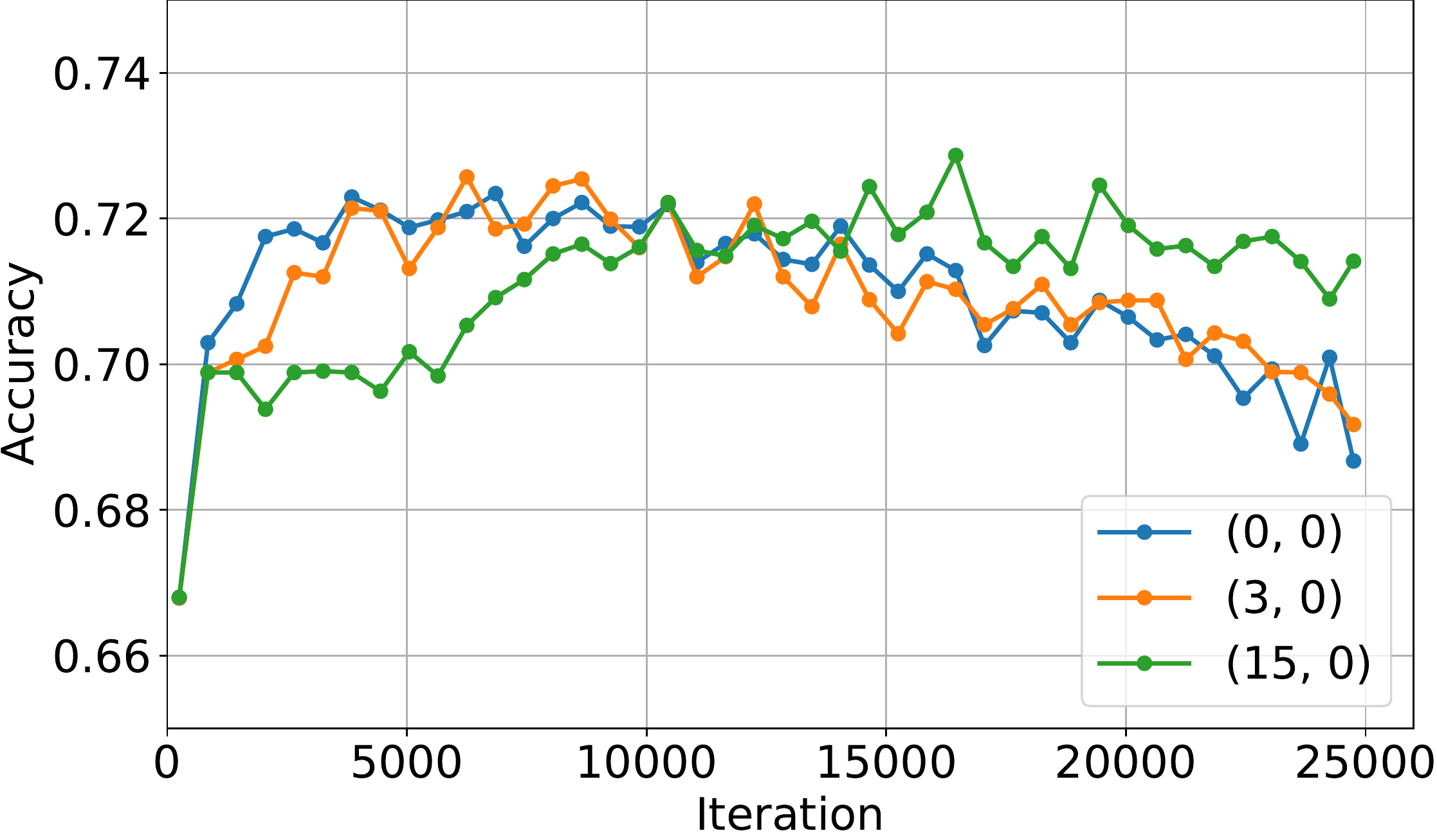}
\par\end{centering}
}\hfill{}\subfloat[\label{fig:MSDA-close-src-src}Close target domain $\theta=0.7$,
only $\lambda^{s-s}$ changes]{\begin{centering}
\includegraphics[width=0.3\textwidth]{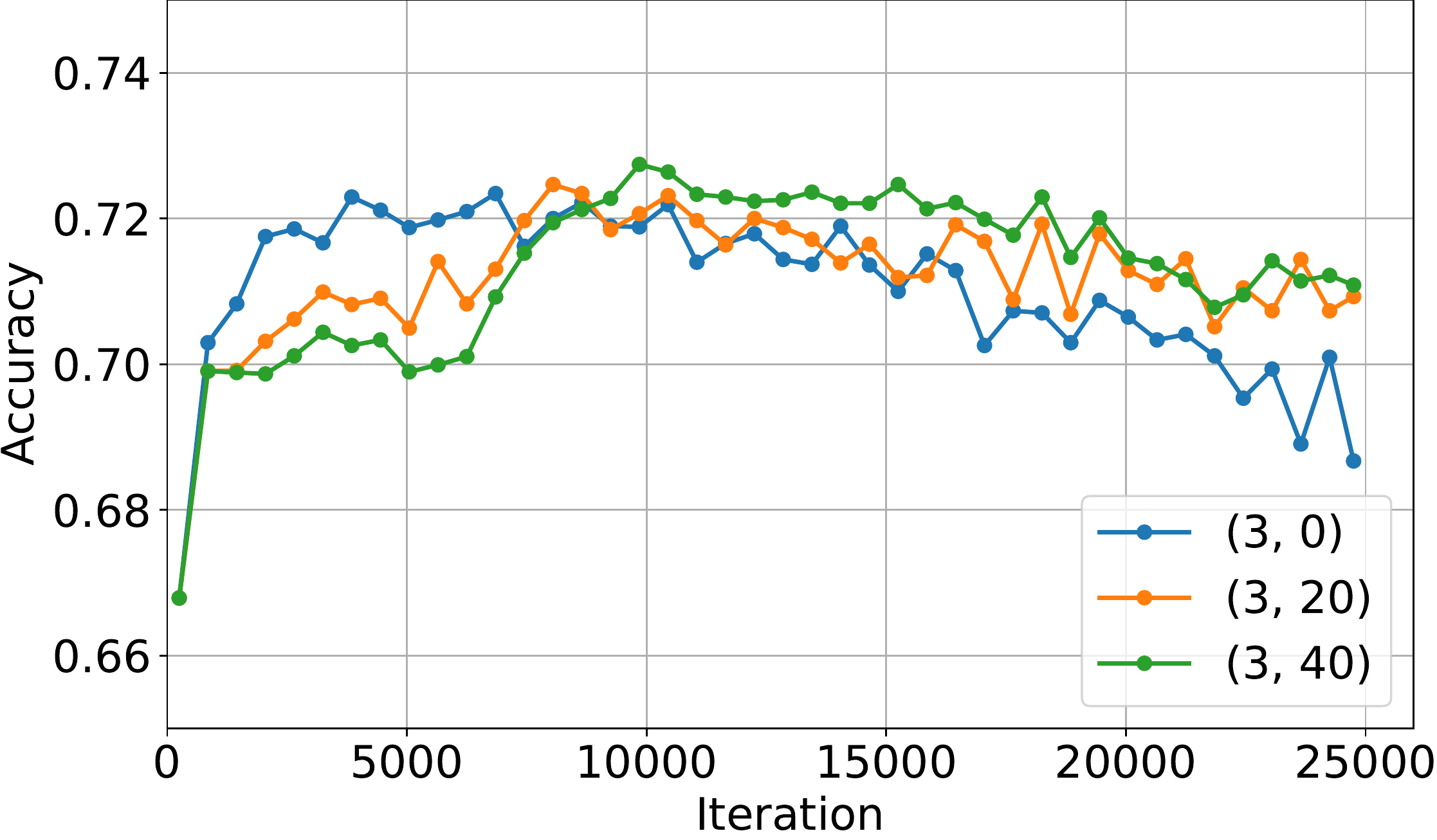}
\par\end{centering}
}
\par\end{centering}
\caption{\label{fig:Accuracy-target-msda}Accuracy on distant and close target
domains, where tuples $\left(\lambda^{s-t},\lambda^{s-s}\right)$
indicate strength of source-target compression and source-source compression,
respectively.}

\end{figure}

\begin{figure}
\begin{centering}
\subfloat[Art Painting as target domain.]{\begin{centering}
\includegraphics[width=0.4\textwidth]{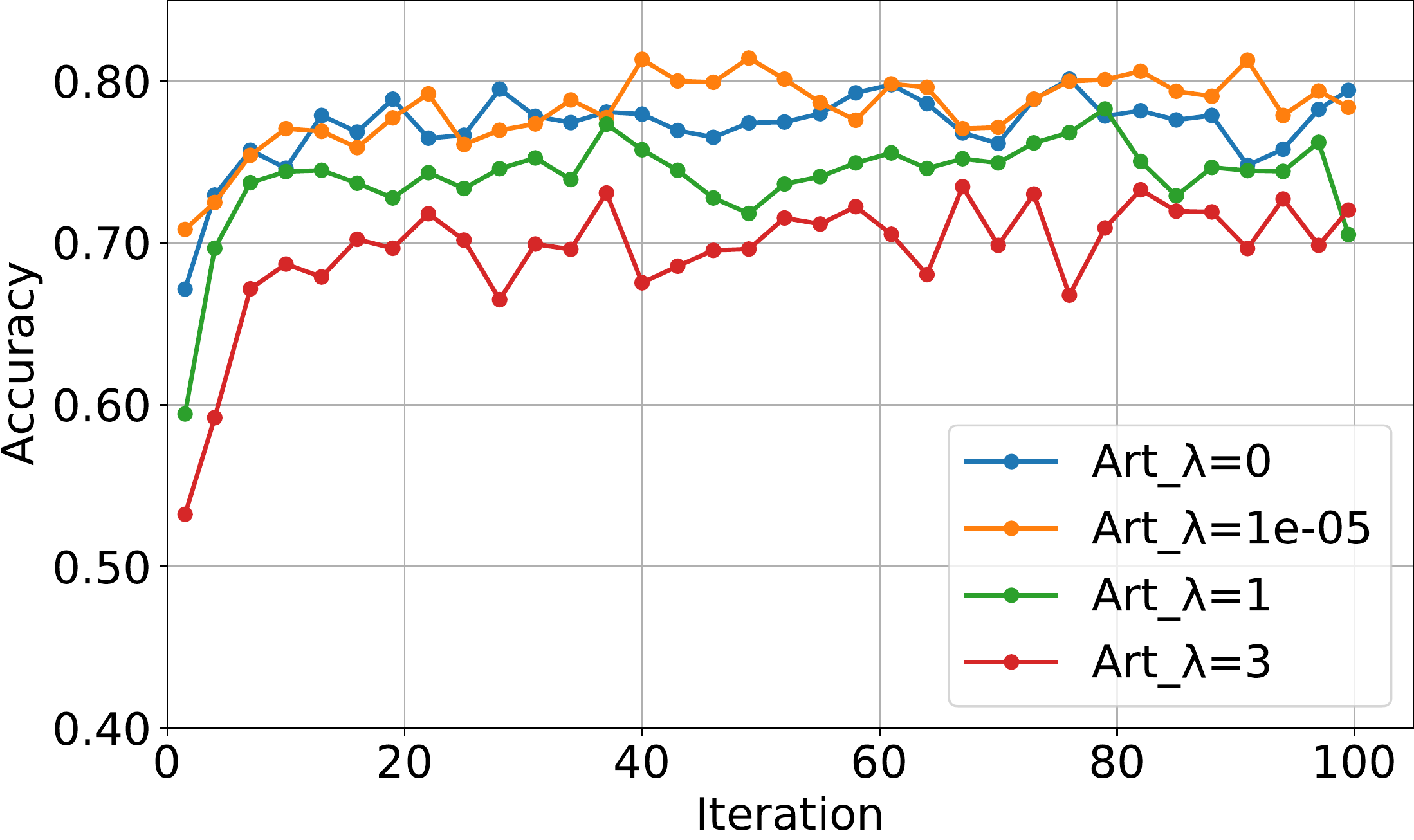}
\par\end{centering}
}\hfill{}\subfloat[Sketch as target domain.]{\begin{centering}
\includegraphics[width=0.4\textwidth]{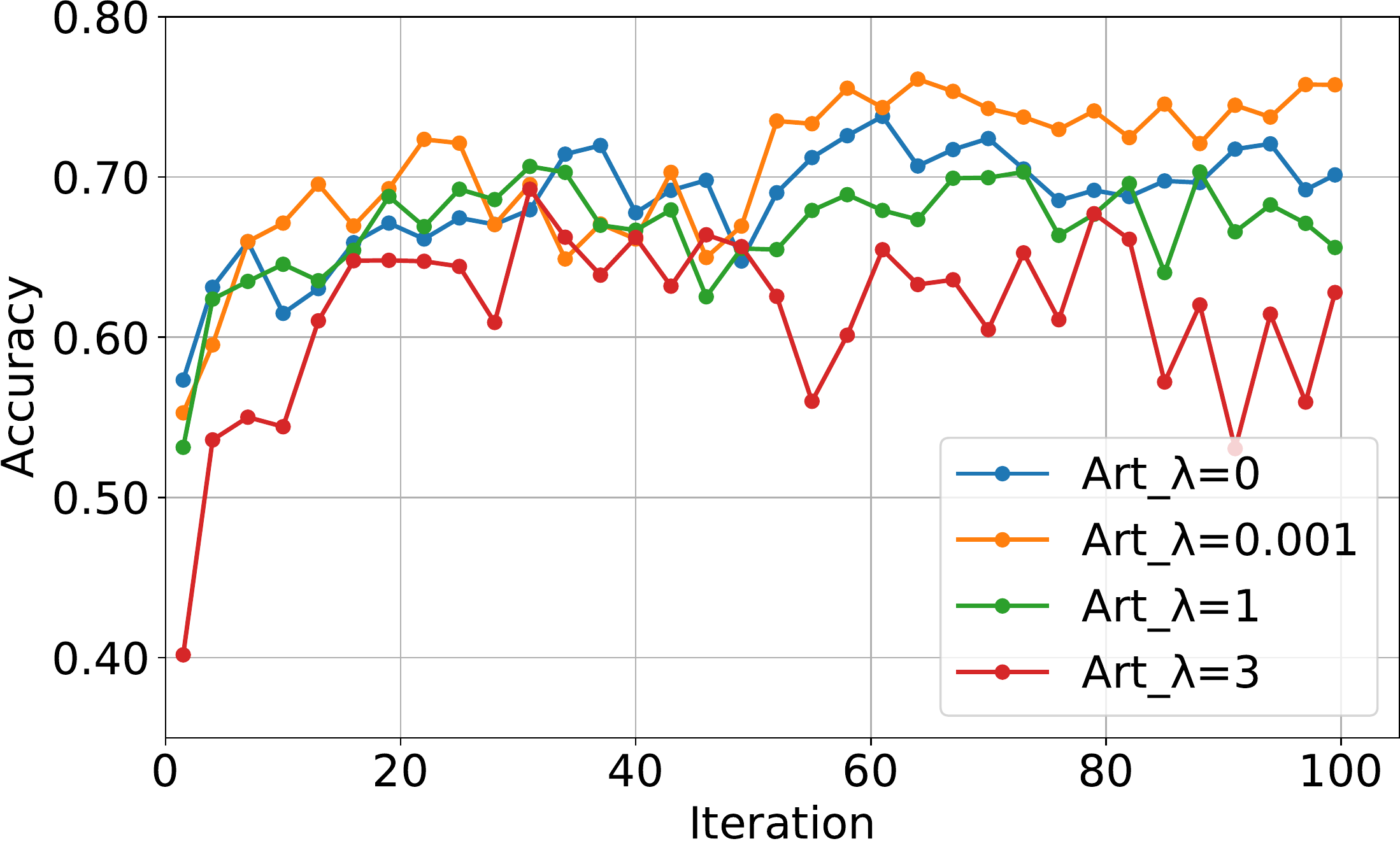}
\par\end{centering}
}
\par\end{centering}
\caption{\label{fig:Accuracy-target-pacs}Domain generalization on PACS with
different compression strength. On both domain, a slight compression
is beneficial which increases accuracy, in-line with possible benefit
of compressed DI representation. However, larger compression deteriorates
performance, which confirm our discussion using trade-off bound.}
\end{figure}

It is evident from the figure that the more compression on both source-source
and source-target representation, the lower the accuracy. This result
aligns with our previous bound (Theorem 8 in main paper), i.e.,

\[
\begin{aligned}d_{1/2}\left(\mathbb{P}_{\mathcal{Y}}^{\pi},\mathbb{P}_{\mathcal{Y}}^{T}\right)-\sum_{i=1}^{K}\sum_{j=1}^{K}\frac{\sqrt{\pi_{j}}}{K}d_{1/2}\left(\mathbb{P}_{g}^{T},\mathbb{P}_{g}^{S,i}\right)-\sum_{i=1}^{K}\sum_{j=1}^{K}\frac{\sqrt{\pi_{j}}}{K}d_{1/2}\left(\mathbb{P}_{g}^{S,i},\mathbb{P}_{g}^{S,j}\right)\\
\leq\left[\sum_{i=1}^{K}\pi_{i}\mathcal{L}\left(\hat{h}\circ g,f^{S,i},\mathbb{P}^{S,i}\right)\right]^{1/2}+\mathcal{L}\left(\hat{h}\circ g,f^{T},\mathbb{P}^{T}\right)^{1/2}.
\end{aligned}
\]

When source-source and source-target compression is applied, the term
$\sum_{i=1}^{K}\sum_{j=1}^{K}\frac{\sqrt{\pi_{j}}}{K}d_{1/2}\left(\mathbb{P}_{g}^{T},\mathbb{P}_{g}^{S,i}\right)+\sum_{i=1}^{K}\sum_{j=1}^{K}\frac{\sqrt{\pi_{j}}}{K}d_{1/2}\left(\mathbb{P}_{g}^{S,i},\mathbb{P}_{g}^{S,j}\right)$
is minimized, raising the lower bound of the loss terms. Subsequently,
as source loss $\sum_{i=1}^{K}\pi_{i}\mathcal{L}\left(\hat{h}\circ g,f^{S,i},\mathbb{P}^{S,i}\right)$
is minimized, the target loss $\mathcal{L}\left(\hat{h}\circ g,f^{T},\mathbb{P}^{T}\right)$
is high and hence performance is hindered.

However, the drop in accuracy for large compression on close target
domain is not as significant as in distant target domain, as indicated
in Figure \ref{fig:MSDA-close-src-tar} and \ref{fig:MSDA-close-src-src}.
In fact, accuracy for some compressed representation is higher than
no compressed representation at larget iteration. It is possible that
the negative effect of raising lower bound as in trade-off Theorem
is counteracted by the benefit of compressed DI representation (Section
2.3.3 of main paper), i.e., the learned classifier for compressed
DI representation better approximates the ground-truth labeling function.
On the other hand, model with general DI representation cannot approximate
this ground-truth labeling function as accurately, but overfit to
training dataset, resulting in target accuracy drop at larger iteration.

\subsection{Experiment on PACS dataset}

In order to verify our theoretical finding on real dataset, we conduct
further experiment on PACS dataset \citep{DG_li2017_iccv_pacs}, which
has 4 domains: Photo, Art Painting, Cartoon, and Sketch. Among the
4 domains, Photo and Cartoon are chosen as training domains, while
Art Painting is chosen as target domain close to the training ones,
and Sketch is the target domain distant from the training ones. We
use Resnet18 \citep{he_resnet} as the feature map, while label classifier
and domain discriminator are multi-layer perceptrons. We only investigate
DG setting on this dataset, in which training objective function is
Eq. \ref{eq:exp_dg_loss}. The result in Figure \ref{fig:Accuracy-target-pacs}
illustrates similar pattern to DG experiment on Colored MNIST. Specifically,
the accuracies for both target domains raise until a peak is reached
and then decrease, which confirms our developed theory for benefit
and trade-off of compressed DI representation.

\end{document}